\documentclass[12pt]{article}

\usepackage[margin=1in]{geometry}
\usepackage{hyperref}
\usepackage{mathrsfs}
\usepackage{enumitem}
\usepackage{amsmath}
\usepackage{amssymb}
\usepackage{mathtools}
\usepackage{amsthm}
\usepackage{float}
\usepackage{tabularx}
\usepackage{algorithm}
\usepackage{makecell}
\usepackage{algpseudocode}
\usepackage{microtype}
\usepackage{graphicx}
\usepackage{authblk}
\usepackage{array}
\usepackage{adjustbox}
\usepackage{subcaption}
\usepackage{longtable}
\usepackage{bm}
\usepackage[table]{xcolor}
\usepackage{booktabs}
\usepackage{placeins}
\usepackage[capitalize,noabbrev]{cleveref}

\usepackage{tikz}
\usetikzlibrary{calc}

\usepackage{siunitx}
\usepackage{amssymb}            
\usepackage{mathtools}          
\usepackage{mathrsfs,float}           
\usepackage{graphicx,cleveref}           
\usepackage{subcaption}         
\usepackage[space]{grffile}     
\usepackage{url}                
\usepackage{lipsum,hyperref}             

\usepackage{amsmath,amsthm,amssymb}
\usepackage{mathtools,systeme}
\usepackage{setspace,dsfont,hyperref}
\usepackage{color}

\newtheorem{thm}{Theorem}[section]

\newtheorem{lem}[thm]{Lemma}
\newtheorem{prop}[thm]{Proposition}
\newtheorem{ex}[thm]{Example}

\theoremstyle{definition}

\theoremstyle{definition}

\theoremstyle{remark}
\newtheorem{rem}{Remark}[section]

\DeclarePairedDelimiter\abs{\lvert}{\rvert}
\DeclarePairedDelimiter\norm{\lVert}{\rVert}
\newcommand{\im}{{\mathbf{i}}}

\algrenewcommand\algorithmicrequire{\textbf{Input:}}
\algrenewcommand\algorithmicensure{\textbf{Output:}}

\usepackage{titlesec}
\titleformat{\section}{\normalfont\bfseries\large}{\thesection}{1em}{}
\titleformat{\subsection}{\normalfont\bfseries}{\thesubsection}{1em}{}

\title{\bfseries Reinforcement Learning for Power-Flow Network Analysis}
\author[1]{Alperen A. Erg\"ur}
\author[2]{Julia Lindberg}
\author[3]{Vinny Miller}

\affil[1]{The University of Texas at San Antonio, Mathematics and Computer Science Departments}
\affil[2]{Georgia Tech, Mathematics Department}
\affil[3]{The University of Texas at San Antonio, Computer Science Department}

\date{}

\begin{document}
\maketitle
\begin{abstract}
The power flow equations are non-linear multivariate equations that describe the relationship between power injections and bus voltages of electric power networks. 
Given a network topology, we are interested in finding network parameters with many equilibrium points.
This corresponds to finding instances of the power flow equations with many real solutions. 
Current state-of-the art algorithms in computational algebra are not capable of answering this question for networks involving more than a small number of variables.  To remedy this, we design a probabilistic reward function that gives a good approximation to this root count, and a state-space that mimics the space of power flow equations.   We derive the average root count for a Gaussian model, and use this as a baseline for our RL agents.  The agents discover instances of the power flow equations with many more solutions than the average baseline.  This demonstrates the potential of RL for power-flow network design and analysis as well as the potential for RL to contribute meaningfully to problems that involve complex non-linear algebra or geometry. \footnote{Author order alphabetic, all authors  contributed equally.}
\end{abstract}

\section{Introduction}

\subsection{Power Flow Networks and Their Non-Linear Equations}

\subsubsection*{Power Networks and the Power Flow Problem} \label{subsubsec:power-flow-problem}

One of the greatest feats of modern engineering is the electrification of our cities across the globe. This was done through the construction and maintenance of electric power grids. The power flow equations are a set of nonlinear equations that describe how electricity behaves as it travels across these grids and are ubiquitous in all studies of electric power networks. Real solutions to the power flow equations give operating points for electric power grids. For a given power network, the problem of finding all real solutions to the power flow equations is called the power flow problem.

In real-world situations, engineers typically only need to find one real solution to the power flow equations. However, there are certain types of problems where finding multiple solutions is desirable. One important such problem is stability analysis, often referred to as dynamic security assessment. The dynamic security assessment (DSA) of a power system evaluates the system's ability to endure all credible contingencies under specific conditions. For direct DSA methods, the multiple real solutions of the power flow equations correspond to equilibrium points of the underlying dynamical system. These equilibrium points are partitioned into stable equilibrium points (SEPs) and unstable equilibrium points (UEPs). 

Stability is usually assessed using a small-signal linearized model, which is only applicable within a small region around the operating point. In contrast, direct methods based on energy functions extend large-signal stability analysis to a broader region around a stable equilibrium point (SEP). These methods are used to determine if a dynamic trajectory stays within the region of attraction of the SEP. Notably, the boundary of this region is defined by the stable manifolds of nearby unstable equilibrium points (UEPs) \cite{pai1981,tsolas1985,zaborszky1988}, and specific UEPs are utilized to calculate energy thresholds for stability. 
\subsubsection*{The Power Flow Equations}

An $n$-node electric power network is modeled as an undirected graph, $G = (V,E)$, where each vertex $v_k \in V$ represents a node (bus) in the power network. There is an edge, $e_{km}$ between vertices $v_k$ and $v_m$ if the corresponding nodes in the power network are connected.

Each vertex $v_k \in V$ has an associated complex power injection, $P_k + \im Q_k$, and each edge $e_{km} \in E$ has a known complex admittance, $b_{km} + \im g_{km}$, where $\im = \sqrt{-1}$.
If there is no edge between vertices $v_k$ and $v_m$ we take the admittance to be zero.

Based on Kirchhoff’s and Ohm’s Law, the power flow equations balance the supply and demand of electricity at every node in the grid. Specifically, at each node, $v_k \in V$, the relationship between the active (real) power, $P_k$, and reactive (imaginary) power, $Q_k$ is captured by the nonlinear equations
\begin{align}
    P_k &= \sum_{m=0}^{n-1} |V_k||V_m| (g_{km} \cos(\theta_k - \theta_m) + b_{km} \sin (\theta_k - \theta_m) ) \label{trigpfeqs1} \\
    Q_k &= \sum_{m=0}^{n-1} |V_k||V_m| (g_{km} \sin (\theta_k - \theta_m) - b_{km} \cos(\theta_k - \theta_m)) \label{trigpfeqs2}
\end{align}
where $V_k$ is the complex voltage magnitude and $\theta_k$ represents the complex voltage angle at node $v_k$.

In traditional power flow studies, two of the parameters $(P_k, Q_k, V_k, \theta_k)$ are assumed known while the remaining are solved for. In particular, there are three types of nodes dictating the equations considered: 
\begin{enumerate}
    \item \textbf{PV:} The active power injections $P_k$ and voltage magnitudes $V_k$ are assumed known, while the reactive power injection $Q_k$ and voltage angles $\theta_k$ are unknown. These model generator buses.
    \item \textbf{PQ:} The active and reactive power injections $P_k, Q_k$ are assumed known while the voltage magnitude $V_k$ and voltage angles $\theta_k$ are unknown. These model load buses.
    \item \textbf{Slack: } The voltage angle and magnitude is known while the active and reactive power injections are unknown. We always assume each network has a slack bus and without loss of generality, we take it to be $v_0$ and assume $\theta_0 = 0$ and $|V_0| = 1$.
\end{enumerate}

For simplicity, we will restrict our study to the case of all PV nodes and one slack bus, although the methods developed in this paper apply to PQ nodes as well. In practice $g_k \ll b_k$ so we assume $g_k = 0$. This is known as a \textit{lossless} network. We make \eqref{trigpfeqs1}-\eqref{trigpfeqs2} algebraic by considering the change of variables $x_k = |V_k| \cos (\theta_k)$ and $y_k =| V_k| \sin ( \theta_k)$. 



By rescaling the variables $x_k, y_k$ we can assume that $V_k = 1$ for all $k = 0,\ldots,n-1$. Typically we assume that the slack bus $v_0$ has $(x_0, y_0) = (1,0)$. 
In our case, we will relax this slightly and assume that at $v_0$ we have equations: 
\begin{align*}
    1 &= x_0^2 + y_0^2, \qquad 1 = x_0^2.
\end{align*}

For $k=1,\ldots, n-1$, the power flow equations for the network $G$ is the system of $2n$ equations in the $2n$ unknowns $(x_0,\ldots,x_n,y_0,\ldots,y_n):$
\begin{align}
P_k &= \sum_{m=0}^{n-1} b_{km}(x_m y_k - x_k y_m) \label{pfeq1}\\
   1 &=  x_k^2 + y_k^2
     \label{pfeq2} \\
     1 &= x_0^2. \label{pfeq4} \\
     1 &= x_0^2 + y_0^2 \label{pfeq3} 
\end{align}

For $k = 1,\ldots, n$, equations \eqref{pfeq1}-\eqref{pfeq3} are the \textit{power flow equations} and real solutions to the power flow equations are \emph{operating points} of the electric power network.  

We claim that \eqref{pfeq1}-\eqref{pfeq3} can be cast as the intersection of $2n$ ellipsoids in $\mathbb{R}^{2n}$. Consider the vector of variables $z = (x_0,\ldots,x_{n-1},y_0,\ldots,y_{n-1})$. Then \eqref{pfeq1}-\eqref{pfeq3} can be rewritten as $z^T A_k z = 1$ for all $1 \leq k \leq 2n$. A detailed reference for this construction is \cite{lesieutre2015an}, but we give a brief outline below.

Since \eqref{pfeq1}-\eqref{pfeq3} are homogeneous quadratic forms, \eqref{pfeq1} and \eqref{pfeq4} can be written as $z^T Q_k z = 1$ for some symmetric matrix $Q_k$ and \eqref{pfeq2} and \eqref{pfeq3} can be written as $z^T S_k z =1$ for $0 \leq k \leq n-1$.

Observe that in this formulation $S_k$ is a matrix of all zeros except in the $(k+1,k+1)$ and $(k+n,k+n)$ entries where the entry is one. In addition, $S_0 + \ldots + S_{n-1} = I_{2n}$. Therefore, instead of $z^T Q_k z = 1$ and $z^T S_k z = 1$, we define 
\begin{align*}
    A_k &= \alpha_{0,k} S_0 + \ldots + \alpha_{n-1,k} S_{n-1} \qquad \text{ if } \qquad  0 \leq k \leq n-1 \\
    B_k  &= \beta_{0,k} S_0 + \ldots + \beta_{n-1,k} S_{n-1} + \gamma_{0,k} Q_0 + \ldots + \gamma_{n-1,k} Q_{n-1} \qquad \text{ if }  \qquad n \leq k \leq 2n-1,
\end{align*}
where $\alpha_{ij}, \gamma_{ij} \in \mathbb{R}_{>0}$ and $\beta_{ij} \in \mathbb{R}$ are chosen such that each $A_k, B_k \succ 0$ and the linear transformation they define is full rank.

In this set up, the matrices $A_k, B_k$ inherit quite a bit of structure. Namely, each $A_k$ is a diagonal matrix with same value in entries $(i,i)$ and $(n+i-1,n+i-1)$. Similarly, each $B_k$ has the same value in entries $(i,i)$ and $(n+i-1,n+i-1)$. In addition, each $B_k$ has an entry in the $(i,j)$ and $(j,i)$ entries if and only if there is an edge $e_{i+1,j+1} \in E$.

\begin{ex}\label{ex: pf eqs}
    To illustrate, we consider an extremely simplified network: A tree with three nodes, $v_0,v_1,v_2$ where the edge $(v_0,v_1)$ has weight $b_{01}$ and the edge $(v_1,v_2)$ has weight $b_{12}$.

Here the power flow equations are
\begin{align*}
    P_1 &= b_{01}(x_0y_1 - x_1 y_0) + b_{12}(x_2 y_1 - x_1 y_2) \\
    P_2 &= b_{12}(x_1 y_2 - x_2 y_1) \\
    1&= x_1^2 + y_1^2 \\
    1 &= x_2^2 + y_2^2 \\
    1 &= x_0^2  \\
    1 &=x_0^2 + y_0^2
\end{align*}
The matrix form of these equations are included in \Cref{detailedeq}. Any linear combination of these matrices will yield a matrix with the following sparsity pattern:

\[ \begin{bmatrix}
        * & 0 & 0 & 0 & * & 0 \\
        0 & * & 0 & * & 0 & * \\
        0 & 0 & * & 0 & * & 0 \\
        0 & * & 0 & * & 0 & 0 \\
        * & 0 & * & 0 & * & 0 \\
        0 & * & 0 & 0 & 0 & *
    \end{bmatrix} \]
\end{ex}

\subsection{Problem Statement and Related Challenges}
We can now concisely state our target problem: For a given network $G=(V,E)$ find the tuple of $n$-matrices $(A_1,A_2,\cdots,A_n)$ that maximizes the number of real solutions to the system of equations in  (\ref{pwfloweq}). 
\begin{equation} \label{pwfloweq} 
\norm{A_1}_2^2 = \norm{A_2}_2^2 = \cdots = \norm{A_n}_2^2 =1  
\end{equation}
There are two main challenges for solving this problem: 
\begin{enumerate}
    \item State-of-the-art computational algebra software scales very poorly in terms of $n$, and this restricts the experimental capabilities to only small instances. That is, given an instance of equations we are only able to solve it if $n$ is small.
    \item Even for small $n$, the state-of-the-art techniques in computational algebra perform poorly for navigating the power flow equation landscape. Consequently, it seems not possible to use these techniques to solve the problem stated above.
\end{enumerate}
In this paper, we address the first technical obstacle by designing a probabilistic reward function that approximates the number of real solutions in a mathematically rigorous way, and scales better than computational algebra algorithms. And the central claim of this paper is that reinforcement learning offers a powerful framework to address the second challenge.
\subsection{Our Contributions}
\begin{itemize}
\item We develop the first machine learning based approach to modeling power-flow equations. Specifically, we formulate the search for power-flow equations with many solutions as a RL problem.
\item We analyze the average-case behavior of these equations as a baseline. This average was unknown prior to this work.
\item We design a novel probabilistic reward function that rigorously  approximates the number of roots of the power-flow equations. Current state-of-the-art in computational algebra can only handle very small instances, whereas our approach is parallelizable and scalable.
\item Our experimental results reveal that RL agents can discover systems of equations yielding substantially more solutions than typically expected. Beyond its direct applications in power-flow network analysis, this success demonstrates reinforcement learning's potential to tackle complex problems in non-linear algebra and geometry. Specifically, the field of real algebraic geometry offers many open conjectures that are now ripe for testing via RL
\end{itemize}
\subsection{Related Work}
To the best of our knowledge, there is no prior machine learning based approach in literature that addresses our target problem. So, we review the application of ML to broader tasks within power-flow system design and modeling. Subsequently, we review  approaches developed by the computational mathematics community.
\subsubsection{Machine Learning in Power Systems}
ML methods, including supervised learning, unsupervised learning, and deep learning, have been used for load and renewable generation forecasting \cite{jain2023loadforecasting,saxena2024intelligent}, state estimation \cite{jongh2022physics}, fault detection \cite{zaborszky1988}, stability assessment \cite{MERIDJI2023108981}, and optimal power flow approximation \cite{van2021machine}. For a more comprehensive review, see \cite{KHALOIE2025125637}.

Reinforcement learning (RL), in particular, has emerged as a promising framework for sequential decision-making problems in power systems \cite{GLAVIC20176918}. By modeling grid operation as a Markov decision process, RL enables data-driven control policies for tasks such as voltage regulation \cite{zhou2021voltage}, frequency control \cite{ADEMOYE201234}, demand response \cite{VAZQUEZCANTELI20191072} and energy storage management \cite{CARDOMIOTA2025125561}. RL is especially attractive in settings with increasing uncertainty due to renewable integration and distributed energy resources, where explicit modeling of all system dynamics may be intractable. However, practical deployment remains challenging due to safety constraints, sample inefficiency, and the need for robust performance under rare but critical contingencies. Recent research therefore focuses on safe RL, constraint-aware formulations, and the integration of RL with optimization-based controllers and domain knowledge to ensure stability, reliability, and interpretability in real-world grid operations. See \cite{su2025saferl} for a review on safe RL in power systems.

\subsubsection{Counting Real Roots for Average-Case}
The study of the expected number of real roots of polynomial systems took off in the $1990$s. While classical results like Bezout’s Theorem tell us that a system of $n$ polynomial equations with degrees $d_1,\ldots, d_n$ can have up to $d_1 \cdots d_n$ isolated complex solutions, the number of real solutions is often much lower than this bound. 
In \cite{shub1993complexity} the authors show that when the coefficients of the polynomial equations are independent, zero mean, Gaussian random variables, then the expected number of real solutions is $\sqrt{d_1 \cdots d_n}$. 
This analysis was later extended to the case of sparse and multihomogeneous polynomial systems \cite{Rojas96onthe, mclennan2002the}. 

Since the power flow equations possess additional structure not present in general polynomial systems (e.g., shared coefficients), the aforementioned results do not apply to our problem. There has been some empirical work on the expected number of real solutions to the power flow equations \cite{zachariah2018distributions, lindberg2023the} as well as the space of network parameters that admit a fixed number of real solutions \cite{lindberg2018the, lesieutre2019on}. 

\subsubsection{Stability Analysis via Root Finding}
Identifying critical unstable equilibrium points (UEPs) on the boundary of the region of attraction has been a persistent challenge. To tackle this issue, several techniques have been proposed for obtaining multiple solutions to the power flow equations. One widely used method is the Newton-Raphson algorithm, which is run with different starting points to find various equilibrium states. However, the convergence of these iterative methods is not always reliable, especially given the potentially fractal nature of the convergence regions \cite{thorp1989}. While there have been improvements in approaches to find certain UEPs \cite{treinen1996,chiang1994,chiang2011,Liu2005,Lee2009,Klump2000}, these methods are not comprehensive and fail to \textit{provably} find all UEPs.

The only provable way to determine all the required UEPs is to calculate all equilibrium i.e. to find all real solutions to the power flow equations. Homotopy and tracing algorithms have been applied \cite{molzahn2013counterexample,tavora1972stability,ma1993an,mehta2016numerical,lesieutre2015an,salam1989parallel,wu_2017, mehta2016recent, lindberg2023the}, but in general, these methods scale poorly\footnote{  The authors in \cite{ma1993an} give an algorithm that scales linearly in the number of real solutions. However, it was later shown that this algorithm may fail to find all real solutions \cite{molzahn2013counterexample}.}.  Recently, monodromy algorithms have been tailored to the power flow equations \cite{lindberg2021exploiting} where all solutions for a network with $20$ nodes were computed. This is the largest realistic situation where all solutions have been provably found. For a general network, finding a stopping criterion for monodromy methods relies on computing a mixed volume \cite{ChenKorchevskaiaLindberg2022Typical}, which is a $\#$ P hard problem \cite{dyer1998on}.

 A fundamental reason for the inefficiency of homotopy methods is that these algorithms have to compute all complex roots before filtering the real ones among them. In addition to the expected number of real roots being much smaller than the number of complex roots, it is also often observed that parameter instances possessing many more real roots than expected are very small in the overall parameter space. Therefore, finding parameter instances with many real roots is of general interest. One such approach to this problem is hill climbing. Hill climbing is a local search technique that iteratively moves in parameter space to parameter values with more real solutions. This greedy approach was developed for a problem in kinematics \cite{dietmaier1998} but has seen success in problems in enumerative algebraic geometry \cite{brysiewicz2021nodes,Brysiewicz2021tangent}. 
 Hill climbing algorithms are local search heuristics, and therefore are highly sensitive to initialization, and they can easily become trapped in suboptimal local maxima corresponding to configurations with fewer real solutions. 

\section{Average Case Analysis} \label{sec:average-case-analysis}
This section will present our mathematical derivation for average case behavior and will prepare the reader for probabilistic reward function design in the next section.  Our randomness model for this section is the following: $A_1, A_2, \cdots, A_n$ are random Gaussian matrices with independent entries centered at zero and variance $\sigma^2$. What is the expected number of solutions to the following system of equations?
\[ \norm{A_1 x}^2 = \norm{A_2 x}_2^2 = \cdots = \norm{A_n x}_2^2 = 1 \]
This estimate will serve as a baseline for our work on finding configurations with many real roots. Surprisingly, this result was not known despite the large body of work on power flow equations.  \cite[Theorem 6.7]{azais2009level}. 
\begin{enumerate}
\item We have random field $Z(t) \in \mathbb{R}^d$ that is created as  $Z(t) = H [Y(t)]$ where  
\[  \{ Y(t) : t \in  W \} \]
 is a Gaussian field indexed by a set $W \in \mathbb{R}^d$ and having values $Y(t) \in \mathbb{R}^n$ for some $n \geq d$, and $H : \mathbb{R}^n \rightarrow \mathbb{R}^d$ is a continuously differentiable function.
\item $Z(t)$ has a density $p_{Z(t)}(x)$ that is continuous function of the pair $(t,x) \in W \times \mathbb{R}^d$
\item $\mathbb{P}\{  \exists t \in W : Z(t) = u \; , \; \det(Z'(t)) = 0  \} = 0 $
\end{enumerate}
Our random field $Z$ is defined as follows: We first define $Y(x)$ as follows: 
\[ Y: \mathbb{R}^n \rightarrow \mathbb{R}^{n^2} \; , \;  Y(x) =  (A_1 x , A_2 x , \ldots, A_n x) \]
Then define a map $H$ as follows:
\[ H : \mathbb{R}^{n^2} \rightarrow \mathbb{R}^n \; , \;  H(y_1, y_2, \ldots, y_n) = (\norm{y_1}^2 , \ldots, \norm{y_n}^2) \]
where $\norm{.}$ denotes the standard Euclidean norm. The random field is then defined as  $Z(x) := H \circ Y(x)$. The expected number of solutions is then given by the following formula:
\begin{equation} \label{formula-prelim}
\mathbb{E} \left (  N(Z, \mathbb{R}^n)  \right)  = \int_{\mathbb{R}^n}  \mathbb{E} \left(  \abs{det(Z'(x))} \; |  \; Z(x) = (1,1,\cdots,1)   \right) \; p_{Z(x)}(1,1,\cdots,1) \; dx   
\end{equation}
Now we need to unravel the density $p_{Z(x)}(u)$:
\[ p_{Z(x)}(1,1,\ldots,1) = p_{H(Y(x))}(1,1,\ldots,1) = P_{H(Y( \norm{x} e_1))}(1,1,\ldots,1)  \]
Where the second equality follows from rotational invariance of the Gaussian field $Y$. Now note that $A_1, A_2, \ldots, A_n$ are independent, so we have
\[   P_{H(Y(\norm{x}e_1))}(1,1,\ldots,1) = P_{H(Y \norm{x} e_1)}(1, \ldots,1) = \left( P_{ \norm{x}^2 \norm{A_1 e_1}^2}(1) \right)^n \]

The random variable $\norm{x}^2 \norm{A_1 e_1}^2 = \norm{x}^2 \left( \psi_1^2 + \ldots + \psi_n^2 \right)$ where $\psi_i \sim \mathcal{N}(0,\sigma^2)$, and $\frac{\norm{x}^2\norm{A_1e_1}^2 }{\norm{x}^2 \sigma^2}$  follows a $\chi_n^2$ distribution.  So we have
\[ P_{\norm{x}^2 \norm{A_1 e_1}^2}(1) = P_{\chi_n^2} (\frac{1}{\sigma^2 \norm{x}^2})  \frac{1}{ \norm{x}^2\sigma^2} \]
and
\begin{equation} \label{density}
p_{Z(x)}(1,1,\ldots,1)   = \left(   \frac{1}{2^{\frac{n}{2}}  \Gamma(\frac{n}{2}) } (   \frac{1}{\sigma^2 \norm{x}^2} )^{\frac{n}{2} -1} e^{-\frac{1}{2 \sigma^2 \norm{x}^2}}     \frac{1}{ \norm{x}^2 \sigma^2}  \right)^n.
\end{equation}
To estimate the formula we need to compute $Z'(x)$. For convenience let us define the following map that sends $x \in \mathbb{R}^n$ to an $n \times n$ matrix $A(x)$ as follows:
\[ A: \mathbb{R}^n \rightarrow \mathbb{R}^{n \times n}  \;  , \;   A(x)^T =  [  A_1^T A_1 x, A_2^T A_2 x, \ldots, A_n^T A_n x   ] \]
Then, we have
\[ Z'(x)= 2 A(x) \; , \;  Z(x) =  A(x) x   \]
Now let $y = Q x $ for an orthogonal matrix $Q$, define $B_i =  A_i Q $, and also extend this definition to the definition of the map $B(x)$ as 
\[ B(x)^T  = [ B_1^T B_1 x, \ldots, B_n^T B_n x ] . \] 
Then, we have $ A(y) y =  B(x) x$, but $A(y) \neq B(x)$. However, we note that every row of $B(x)$ is obtained acting on rows of $A(y)$ with the orthogonal matrix $Q$. So, we naturally have $\det(B(x)) = \det(A(y))$.
Now note that rows and columns of $A_i$ are centered Gaussian random vectors, therefore due rotational invariance of Gaussians the distribution of $A_i$ does not change under the action of $Q$.  This means the following:
\[    \mathbb{E} \left(  \abs{det(Z'(x))} \; |  \; Z(x) = (1,1,\ldots,1 ) \right)   =  \mathbb{E} \left(  \abs{det(Z'(y))} \; |  \; Z(y) = (1,1,\ldots,1)   \right)   \]
Also note that the density formula is invariant under the action of $Q$ as well. This means using an arbitrary choice $x = e_0$ we can simplify as follows:
\begin{equation} \label{formula-pre}
\abs{S^{n-1}} \int_{0}^{\infty}  \mathbb{E} \left(  \abs{det(Z'(t e_1))} \; |  \; Z(t e_1) = (1, \ldots,1 ) \right) \; p_{Z(t e_1)} (1,\ldots,1) t^{n-1} \; dt 
\end{equation}
where $\abs{S^{n-1}}$ denotes the volume of the sphere $S^{n-1}$. Note that  
\[ Z(t e_1)= (1,\ldots,1)  \Leftrightarrow Z(e_1) =  (\frac{1}{t^2}, \ldots, \frac{1}{t^2}) = \left( (A_i A_i^{T})_{11} \right)_{i}, \;
\text{and} 
\;  Z'(te_1) =  t Z'(e_1). \]
Therefore,
\[   \mathbb{E} \left(  \abs{det(Z'(t e_1))} \; |  \; Z(t e_1) = (1,\ldots,1 ) \right) =   t^n \; \mathbb{E} \left(  \abs{det(Z'(e_1))} \; |  \; Z(e_1) = (\frac{1}{t^2}, \ldots, \frac{1}{t^2}) \right)\]

Let the columns of the matrices $A_i$ be $a_{ij}$ for $j=1,2,\ldots,n$. The assumption $Z(e_1) = (\frac{1}{t}, \ldots, \frac{1}{t^2})$ is equivalent to $\norm{a_{i1}}=\frac{1}{t}$ for all $i$. Moreover, we have
\[  Z'(e_1) = 2 
\begin{bmatrix*}
\langle a_{11} , a_{11} \rangle & \langle a_{11} , a_{12} \rangle & \ldots & \langle a_{11} , a_{1n} \rangle \\
\vdots                          &                                  &       &   \vdots                         \\
\langle a_{n1} , a_{n1} \rangle & \langle a_{n1} , a_{n2} \rangle & \ldots & \langle a_{n1} , a_{nn} \rangle \\
\end{bmatrix*}
 \]
That is, the assumption  $Z(e_1) = (\frac{1}{t}, \ldots, \frac{1}{t^2})$ vs  $Z(e_1) = (1, \ldots, 1)$ scales the determinant with $\frac{1}{t^2}. \frac{1}{t^{n-1}}=\frac{1}{t^{n+1}}$. This yields,
\[   \mathbb{E} \left(  \abs{det(Z'(t e_1))} \; |  \; Z(t e_1) = (1,\ldots,1 ) \right) =  t^{-1} \mathbb{E} \left(  \abs{det(Z'(e_1))} \; |  \; Z(e_1) = (1, \ldots, 1) \right) . \]
So, we can further simplify as
\begin{equation}
\abs{S^{n-1}} \; \mathbb{E} \left(  \abs{det(Z'(e_1))} \; |  \; Z(e_1) = (1,\ldots,1 ) \right)   \int_{0}^{\infty} t^{n-2} p_{Z(t e_1)} (1,\ldots,1) \; dt 
\end{equation}

We note that scaling all equations with the same constant does not change the root count, so after this point w.l.o.g we'll assume $\sigma =1$.   
\begin{prop} \label{randomdet}
    \[ \abs{S^{n-1}} \; \mathbb{E} \left(  \abs{det(Z'(e_1))} \; |  \; Z(e_1) = (1,\ldots,1 ) \right)  
  \sim  \pi^{\frac{n-2}{2}}  2^{\frac{3n+1}{2}}  \] 
\end{prop}
\begin{proof}
Conditioning on $Z(e_1)=(1,1,\ldots,1)$ is equivalent to assuming
\[ \norm{A_1 e_1} = \norm{A_2 e_1} = \cdots = \norm{A_n e_1} =1  \]
Notice that $\norm{A_i e_1}^2 = e_1^T A_i^T A_i e_1$ which is simply the norm of first column of $A_i$.
So, we are conditioning on the event that first column of all matrices have norm $1$, which is equivalent to assuming
the first columns being independent uniformly random vector on the unit sphere.
Now note that 
\[ Z'(e_1) = 2 [ A_1^TA_1 e_1  A_2^TA_2 e_1  \cdots A_n^T A_n e_1 ]\]
Observe that the vector $A_i^T A_i e_1$ is conditioned on having first coordinate $1$, and the others being an independent Gaussian vector multiplied with a uniform vector on the sphere.
That is, the vector has first coordinate $1$ and the rest of the coordinates are independent Gaussians with mean zero and variance one. In return, the matrix 
\[ [ A_1^TA_1 e_1  \ A_2^TA_2 e_1  \ \cdots \  A_n^T A_n e_1 ] \]
has first column fixed to be $(1,1,\cdots,1)$ and the rest of the columns are independent Gaussian.  Notice that due to Hadamard determinant formula
\[  \abs{\det \left( [ A_1^TA_1 e_1  A_2^TA_2 e_1  \cdots A_n^T A_n e_1 ] \right)} =  \sqrt{n} \abs{\det( PX)} \]
where $P$ is the projection on the orthogonal complement of $(1,1,\cdots,1)$ and $X$ is $n \times (n-1)$ Gaussian matrix with independent entries and variance $1$.  
The expectation of $\mathbb{E} \abs{\det (PX)}$ is known (cite), we have 
\[ \mathbb{E} \abs{\det (PX)} \sim  \frac{2^{\frac{n-2}{2}} \Gamma(\frac{n-1}{2})}{\sqrt{\pi}} \]
Note that 
 \[ \abs{S^{n-1}}  = \frac{2 \pi^{\frac{n-1}{2}}}{\Gamma(\frac{n}{2})}  \]
 So, we have
 \[ |S^{n-1}| \mathbb{E} \left(  \abs{det(Z'(e_1))} \; |  \; Z(e_1) = (1,\ldots,1 ) \right)   =  2^n \sqrt{n} \pi^{\frac{n-2}{2}} 2^{\frac{n}{2}} \frac{\Gamma(\frac{n-1}{2})}{\Gamma(\frac{n}{2})} \]
\[  \frac{\Gamma(\frac{n-1}{2})}{\Gamma(\frac{n}{2})} \sim \frac{1}{\sqrt{\frac{n-1}{2}}}\]
 
\end{proof} 

Using this claim we can prove the following theorem.
\begin{thm} \label{average-count}
Let $A_1,A_2,\ldots,A_n$ be random Gaussian matrices with independent entries of standard deviation $\sigma > 0$. Then, the expected number of solutions to the following system of equations 
\[ \norm{A_1x}^2 = \norm{A_2 x}^2 = \ldots = \norm{A_n x}^2 =1 \]
is  
 \[ c_1  n^{-\frac{1}{2}} 2^{\frac{n}{2}} . \]
\end{thm}
The  estimate in Thm \ref{average-count} will serve as the baseline for our agents exploration.
Using Proposition \ref{randomdet}, the proof of Theorem \ref{average-count}  then is completed as follows:

\[  p_{Z(t e_0)} (1,\ldots,1)  =  \left( 2^{\frac{n}{2}} \Gamma(\frac{n}{2}) \right)^{-n}    t^{-n^2} e^{-\frac{n}{2 t^2}}   \] 
 
\[  \int_{0}^{\infty} t^{n-2} p_{Z(t e_0)} (1,\ldots,1) \; dt  =  \left( 2^{\frac{n}{2}} \Gamma(\frac{n}{2}) \right)^{-n}    \int_{0}^{\infty}  t^{-2}  t^{-n^2+n} e^{-\frac{n}{2t^2}} \; dt \]

For $u=\frac{1}{t}, \; du = -\frac{1}{t^{2}} \; dt$ we have
 \[  \int_{0}^{\infty}  t^{-2}   t^{-n^2+n} e^{-\frac{n}{2t^2}} \; dt =  \int_{0}^{\infty} u^{n^2-n} e^{-\frac{nu^2}{2}}  \; du  \]

\[ \mathbb{E} \left (  N_u(Z, \mathbb{R}^n)  \right)   \sim  \pi^{\frac{n-2}{2}}  2^{\frac{3n+1}{2}}    \left( 2^{\frac{n}{2}} \Gamma(\frac{n}{2}) \right)^{-n}  \int_{0}^{\infty} u^{n^2-n} e^{-\frac{nu^2}{2}}  \; du  \]

Now notice that the integral expression on the right-hand-side is almost equivalent to $n^2-n$'th moment of $z \sim \mathcal{N}(0,\frac{1}{n})$. Using known formulas for Gaussian moments, this gives
\[ \int_{0}^{\infty} u^{n^2-n} e^{-\frac{nu^2}{2}}  \; du = 2^{\frac{n^2-n-1}{2}} n^{-\frac{n^2-n+1}{2}} \Gamma(\frac{n^2-n+1}{2})  \]
So, we have
\[ \mathbb{E} \left (  N_u(Z, \mathbb{R}^n)  \right)   \sim  \pi^{\frac{n-2}{2}}  2^{n}    \Gamma(\frac{n}{2})^{-n}  n^{-\frac{n^2-n+1}{2}} \Gamma(\frac{n^2-n+1}{2})  \]

Recall that (Stirling): 
   \[  \ln \left( \Gamma(z) \right) \sim (z-\frac{1}{2}) \ln(z) -z + \frac{1}{2} \ln(2 \pi) + \frac{1}{12z} \]
We use this to estimate 
\[   n \ln(2) + \frac{n-2}{2} \ln(\pi) -n \ln (\Gamma(\frac{n}{2})) + \ln (\Gamma(\frac{n^2-n+1}{2})) - \frac{n^2-n+1}{2} \ln(n)     \]
And we get
\[ \frac{n+1}{2} \ln(2) - \frac{1}{2} \ln(n) - \frac{1}{2} \ln (\pi) + \frac{1}{12} \]
This implies,
\begin{equation}
    \mathbb{E} \left (  N_u(Z, \mathbb{R}^n)  \right)   \sim  n^{-\frac{1}{2}} 2^{\frac{n+1}{2}} \frac{e^{\frac{1}{2}}}{\sqrt{\pi}}
\end{equation}

\section{Reward Function Design} \label{rewarddesign}
Counting real roots of systems of quadratic equations for large number of variables is beyond reach out state-of-the-art computational algebra algorithms; this is the true reward for the agents. Using the tools from previous section, we can design a rigorous approximate reward as follows: Suppose the system of equations we are trying to count roots for is $(A_1,A_2,\cdots,A_n)$, we define a collection of random matrices as 
$\tilde{A_i} = A_i + \delta X_i$ where $0 < \delta <1$ is an interpolation parameter, and $X_i$ are random Gaussian matrices with mean zero variance one entries. If we can estimate the expected number of solutions to 
\[ \norm{\tilde{A_1} x}_2^2 = \norm{\tilde{A_2} x}_2^2 = \cdots = \norm{\tilde{A_n}x }_2^2 =1  \]
this can serve as a principled proxy. This is what we manage, but it takes several mathematical insights and practical tricks to do so. This section will be more mathematical, and  \Cref{rewardimp} will be more algorithmic. We start by stating a tool from \cite[Thm 4.5.9]{barvinok2016combinatorics}: 
\begin{lem} \label{scaling}
Let  $M_1,M_2,\ldots, M_n$ be positive definite matrices, then there exists an invertible matrix $Z$ and positive scalars $\alpha_1,\alpha_2,\ldots,\alpha_n$ such that $T_i = \alpha_i Z^T M_i Z$ satisfies
\[ \mathrm{Trace}(T_i) = 1 \; \text{for each} \; i=1,2,\ldots,n, \; \text{and} \; T_1 + T_2 +\ldots + T_n = \mathbb{I}   \]
\end{lem}

We are interested in counting the number of real solutions to the following system of equations:
\begin{equation} \label{target}
   \norm{A_1 x}^2= \norm{A_2 x}^2 = \ldots = \norm{A_nx}^2= 1  \; \text{for} \; i=1,2,\ldots,n. 
\end{equation}
It is clear that the number of solutions to \Cref{target} is equal to the number of solutions to \Cref{targett}
\begin{equation} \label{targett}
   \norm{A_1 Z x}^2= \norm{A_2 Z x}^2 = \ldots = \norm{A_n Zx}^2= 1  \; \text{for} \; i=1,2,\ldots,n. 
\end{equation}
for any invertible matrix $Z$.  Also note that scaling all equations at once does not change the number of solutions. Using Lemma \ref{scaling}, and noting $\mathrm{Trace}(Z^TA_i^TA_iZ)=\norm{A_iZ}_F^2$, we can find a $Z$ such that 
\[ \frac{1}{\norm{A_1Z}_F^2} Z^TA_1^TA_1Z + \frac{1}{\norm{A_2Z}_F^2} Z^TA_2^TA_2Z + \ldots + \frac{1}{\norm{A_nZ}_F^2} Z^TA_n^TA_nZ = \mathbb{I}   \]
Note that for any matrix $A$ solving 
\[ \norm{A x}^2 =1 \; \text{is equivalent to solving} \; \norm{\frac{1}{\norm{A}}A x}^2 = \frac{1}{\norm{A}^2} \]
So, we define $C_i = \frac{1}{\norm{A_iZ}} A_iZ$ \Cref{target}. 
\begin{align*} 
\begin{split}
   \norm{C_i}=1 \; \text{for each} \; i, \; \text{and} \; C_1^T C_1 + C_2^T C_2 + \ldots + C_n^T C_n=\mathbb{I}   \\
     c_i = \frac{1}{\norm{A_i Z}^2} \; \text{for} \; i=1,2,\ldots, n \\ 
    \text{Count} \; x \in \mathbb{R}^n \; \text{such that} \; \norm{C_i x}^2 = c_i 
\end{split}
\end{align*}
Now to make our computations easier, we do the following:  
\begin{enumerate}
    \item Compute  $g := \left( \frac{1}{\norm{A_1Z}^2} + \frac{1}{\norm{A_2Z}^2} + \cdots + \frac{1}{\norm{A_nZ}^2} \right)^{\frac{1}{2}} $.
    \item Do a scaling of variables and the equation by setting $x = \frac{\sqrt{n}}{g} x$. After this, an equation of the form $\norm{A x}^2 = \beta$ becomes $\norm{A x}^2 = \frac{n \beta}{g^2}$.
\end{enumerate}

In summary, we defined $ C_i = \frac{1}{\norm{A_iZ}} A_iZ \;  ,  \; g = \left( \sum_{i=1}^n \frac{1}{\norm{A_iZ}^2} \right)^{\frac{1}{2}}$. Now the system of equations becomes the following:
\begin{align} \label{target-normalized}
\begin{split}
   \norm{C_i}=1 \; \text{for each} \; i, \; \text{and} \; C_1^T C_1 + C_2^T C_2 + \ldots + C_n^T C_n=\mathbb{I}   \\
   (c_1,c_2,\ldots,c_n) \in (\mathbb{R}_{+})^n \; , \; c_i = \frac{n}{g^2 \norm{A_i Z}^2 }  \; . \;  c_1 + c_2 + \cdots + c_n = n \\
    \text{Count} \; x \in \mathbb{R}^n \; \text{such that} \; \norm{C_i x}^2 = c_i.
\end{split}
\end{align}
Our model of randomness is set-up as follows: We use $c=(c_1,c_2,\ldots,c_n)$ to denote our $n$-tuple with $c_i >0$ and $c_1+c_2+\ldots+c_n=n$, and $\mathcal{C}=(C_1,C_2,\ldots,C_n)$ such that $C_1^T C_1 + C_2^T C_2 + \ldots + C_n^T C_n=\mathbb{I}$. Then we have random Gaussian matrices $A_i$ that are defined as 
\begin{align} \label{relax-normalized}
\begin{split}
 A_i = C_i + \delta X_i \; \text{where} \; X_i \; , \;
\norm{A_1x}^2 = c_1 \; , \; \norm{A_2x}^2 = c_2 \; , \cdots , \; \norm{A_nx}^2= c_n  
\end{split}
\end{align}
where $X_i$ has independent $\mathcal{N}(0,1)$ entries. We will use the expected number of solutions to \Cref{relax-normalized} as a proxy to exactly counting the number of real solutions of \Cref{target-normalized}. Now suppose we have solution $x \in \mathbb{R}^n$ to \Cref{relax-normalized}, then
\begin{equation} 
   n =  \sum_{i=1}^n x^T A_i^T A_i x =  \norm{x}^2 +  x^T \left( \delta \sum_{i=1}^n (X_i^T C_i + C_i^T X_i)  + \delta^2 \sum_{i=1}^n X_i^T X_i \right) x  
\end{equation}
Note that  $X := \delta \sum_{i=1}^n (X_i^T C_i + C_i^T X_i)  + \delta^2 \sum_{i=1}^n X_i^T X_i $  is a real symmetric matrix, and we have $n = \norm{x}^2 + x^T X x$. Moreover,   $X \sim \delta^2 n \; \mathbb{I}$ with high  probability due to very strong concentration properties of Gaussians. Therefore, with high probability, we have
\begin{equation} \label{Gaussloci}
    n \sim (1 + \delta^2 n) \norm{x}^2. 
\end{equation}
Now we can give a snapshot of our plan for implementation of this reward function: Lemma \ref{scaling} can be formulated as a convex optimization program, and solved using quasi-Newton methods efficiently. After completing the scaling, we perform a monte-carlo approximation of the Kac-Rice formula corresponding to \Cref{relax-normalized} where $x$ will be sampled from the region with $ n(1-\delta) \leq \norm{x}^2 \leq n(1+\delta)$. The computation of the conditional expectation in Kac-Rice formula at a given sample $x$ is cumbersome, we will use a trick from \cite{feliu2022kac} alongside importance sampling to simplify that computation. These steps are explained at \Cref{rewardimp}.

\section{Reinforcement Learning Setup}
With the theoretical foundation for our probabilistic root-counting proxy established in previous section, the central challenge shifts to effectively navigating the complex parameter space of the power-flow equations.  By mapping the system matrices to the environment's state space and bounding the allowed perturbations as actions, the agent learns to iteratively reshape the network parameters to maximize the expected number of real solutions. The remainder of this section details the actor-critic architecture used to traverse this space and the precise computational implementation of the reward function.
\subsection{Problem formulation and Actor-Critic Architecture}
We formulate the search for power-flow equations with many real solutions as a reinforcement learning problem using a twin-delayed actor-critic architecture \cite{fujimoto2018addressing}. and verify the outcomes in experiments with small instances with Julia Homotopy \cite{juliahomotopy}. The reinforcement learning environment is defined as follows
\begin{itemize}
    \item  State Space: The state-space is modeled as the collection of all $n\times n$ matrices with entries in the range $[-1, 1]$. This representation is valid because scaling all entries of all matrices in the power-flow equations by the same scalar does not change the solutions.
    \item Action Space: While the action-space is theoretically the same as the state-space, actions are restricted to perturbations of a capped size to ensure robust updates. Specifically, the agent can only modify a matrix entry by at most $\hat{a}$ during a single step.
    \item Reward: At each step, the agent outputs a system of matrices $\mathcal{A}=(A_{1},...,A_{n})$. The reward is then calculated by running the probabilistic Monte Carlo approximation on this system to estimate the real solution count.
    \item Episode Structure: Because the reward function relies on an approximation, defining a strict terminal condition is difficult. Therefore, the episode length is treated as a hyperparameter.
    \item Objective: The agent's goal is to improve some starting system of matrices $\mathcal{A}_s=(A_1^s,\ldots,A_n^s)$ over a sequence of steps and arrive at a system of matrices $\mathcal{A}_f=(A_1^f,\ldots,A_n^f)$ such that $ \norm{A_1^f x}_2^2 = \cdots = \norm{A_n^f x}_2^2 = 1$ has many more real solutions than $\norm{A_1^s x}_2^2 = \cdots = \norm{A_n^s x}_2^2 = 1 $.

\end{itemize}

\subsection{Reward Function Implementation} \label{rewardimp}
Lemma \ref{scaling} can be proved by minimizing function $f: \mathbb{R}^n \rightarrow \mathbb{R}$ defined below, where $Q_1, \dots, Q_n$ are $n \times n$ positive definite matrices.
\[ f (t_1, \dots, t_n) = \ln \det (e^{t_1} Q_1 + \dots + e^{t_n} Q_n) \; \text{s.t.} \quad t_1 + \dots + t_n = 0 \]
This minimization gives us the $\alpha_i$'s and $Z$ that satisfy Lemma \ref{scaling} such that $\alpha_i = e^{t_i}$ and for $S = e^{t_1}Q_1 + \dots + e^{t_n}Q_n$, $Z = S^{-1/2}$. The function $f$ is shown to be convex in \cite[Lemma 4.5.7]{barvinok2016combinatorics}, so by eliminating the constraint with setting $t_n = -t_1 - \dots -t_{n - 1}$ we can use first-order methods.
For $t_1, \ldots, t_{n - 1}$, let $E_n = e^{-t_1 - \dots -t_{n - 1}} Q_n$ and $S = e^{t_1} Q_1 + \dots + e^{t_{n - 1}} Q_{n - 1} + E_n$.
The gradient is
\[ \nabla f(t_1, \ldots, t_{n - 1}) = \nabla \ln \det (S)  = 
    \begin{bmatrix}
        \mathrm{Trace}(S^{-1}(e^{t_1} Q_1 - E_n)) \\
        \mathrm{Trace}(S^{-1}(e^{t_2} Q_2 - E_n)) \\
        \vdots \\
        \mathrm{Trace}(S^{-1}(e^{t_{n - 1}} Q_{n - 1} - E_n))
    \end{bmatrix}
\]

To connect this back to our setup, let $\mathcal{A}=(A_1,\ldots,A_n)$ be an unnormalized system of matrices.
In place of the positive definite matrices $Q_1, \dots, Q_n$ we will use positive semidefinite matrices $A_1^T A_1, \dots, A_n^T A_n$.
The normalization procedure will return matrices $T_i = \alpha_i Z^T A_i^T A_i Z$, where we note that $c_i = \frac{1}{\norm{A_i Z}^2} = \frac{1}{\mathrm{Trace}(Z^T A_i^T A_i Z)} = \alpha_i$. Though Barvinok's normalization procedure assumes positive definite matrices and we can only guarantee positive semidefinite matrices, our normalization calculations never ran into an undefined $f$.
Gurvits and Samorodnitsky \cite[Theorem 1.10]{Gurvits2002} give a condition that a system of positive semidefinite matrices must meet for the normalization procedure to work.

With all this in hand, we use the BFGS method \cite{nocedal2006numerical}  via SciPy to do the unconstrained optimization. For each system size $n$, 50 systems were generated and normalized, and the plots show the average results for all 50 generated systems. A system $\mathcal{A}=(A_1,\ldots,A_n)$ was generated by creating $n$ Gaussian random matrices with independent $\mathcal{N}(0,1)$ entries.
\cref{fig:bfgs-time} gives the time taken to normalize the systems, and \cref{fig:bfgs-trace} and \cref{fig:bfgs-summation} give how close the normalized system is to the Lemma \ref{scaling} conditions.
\cref{fig:bfgs-summation} shows the euclidean distance between the system-summation matrix, so $A_1^T A_1 + \ldots + A_n^T A_n$, and $\mathbb{I}$.
\cref{fig:bfgs-trace} gives the euclidean distance between the vector formed by taking the traces of our normalized matrices, so $[ \mathrm{Trace}(A_1^T A_1), \ldots, \mathrm{Trace}(A_n^T A_n) ]$, and the 1-vector. As far as accuracy is concerned, the trace distance is the metric to be wary of given its seesaw growth toward $1\mathrm{e}{-6}$ in the second half of \cref{fig:bfgs-trace}.

\begin{figure}
\centering
\begin{subfigure}{.45\textwidth}
  \centering
  \includegraphics[width=.8\linewidth]{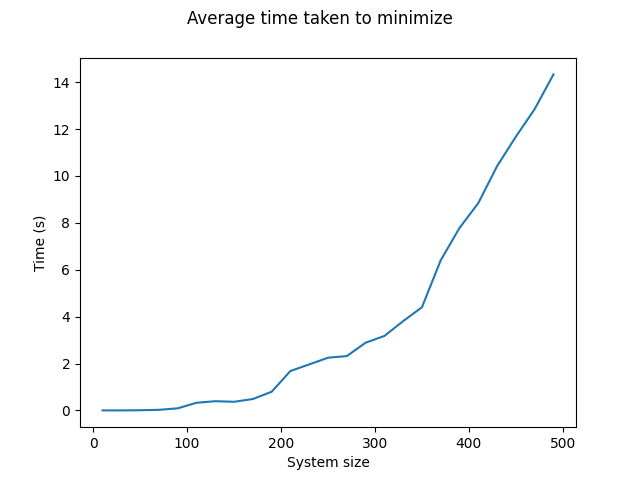}
  \caption{The average amount of time taken to normalize the system of matrices.}
  \label{fig:bfgs-time}
\end{subfigure}%
\hfill
\begin{subfigure}{.45\textwidth}
  \centering
  \includegraphics[width=.8\linewidth]{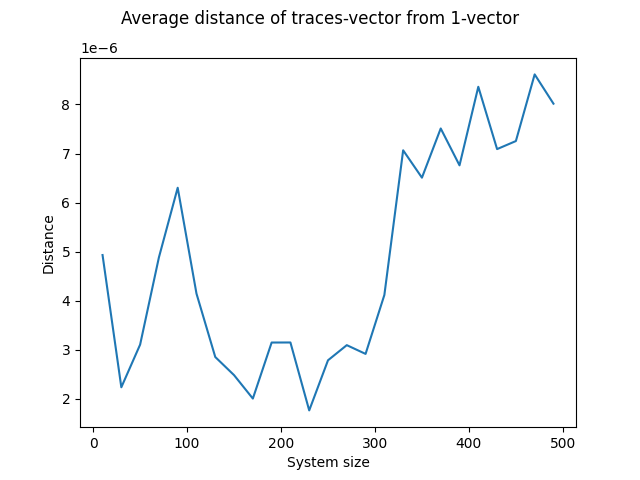}
  \caption{The average euclidean distance of the system matrix traces from the 1-vector.}
  \label{fig:bfgs-trace}
\end{subfigure}
\hfill
\begin{subfigure}{.6\textwidth}
  \centering
  \includegraphics[width=.8\linewidth]{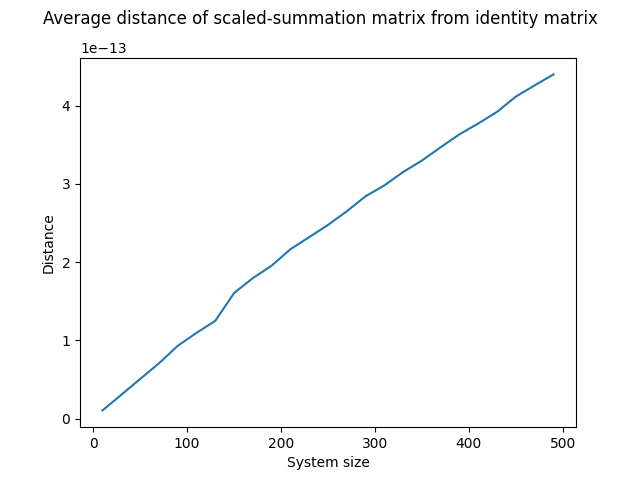}
  \caption{The average euclidean distance of the system matrices summation from the identity matrix.}
  \label{fig:bfgs-summation}
\end{subfigure}
\label{fig:bfgs-scaling}
\end{figure}
\begin{table}[htbp]
\centering
\resizebox{\linewidth}{!}{%
\begin{tabular}{l|rr|rr|rr|c}
\cline{2-7}
 & \multicolumn{2}{c|}{\textbf{Time Taken (s)}} & \multicolumn{2}{c|}{\textbf{Trace Distance}} & \multicolumn{2}{c|}{\textbf{Summation Distance}} & \\ \hline
\textbf{System Size} & \textbf{BFGS} & \textbf{Newton-CG} & \textbf{BFGS} & \textbf{Newton-CG} & \textbf{BFGS} & \textbf{Newton-CG} & \textbf{Success Rate} \\ \hline
50  & 0.0079 & 0.11 & 1.3e-5 & 8.9e-6 & 5.0e-14 & 5.0e-14 & 0.3  \\ \hline
150 & 0.37   & 10   & 7.6e-7 & 5.8e-7 & 1.6e-13 & 1.6e-13 & 0.08 \\ \hline
250 & 2.3    & 82   & 8.8e-6 & 7.5e-6 & 2.5e-13 & 2.5e-13 & 0.2  \\ \hline
\end{tabular}%
}
\caption{Accuracy improvements made by using a Newton-conjugate corrector for 5 steps. The same systems used in \cref{fig:bfgs-scaling} were used here. Only those systems that were successfully corrected were averaged over.}
\label{table:newton-cg}
\end{table}
In an attempt to improve the trace distance we used the Newton-conjugate second-order optimization method \cite{nocedal2006numerical} for five steps after BFGS, but the increase in computation time was not worth the small accuracy improvements as \cref{table:newton-cg} demonstrates.
The computation time increase came from $f$'s expensive Hessian
\[ \frac{\partial^2 f}{\partial t_i \partial t_j}(x) = -\mathrm{Trace}(S^{-1} (e^{x_i}Q_i - E_n) S^{-1} (e^{x_j}Q_j - E_n)) + \mathrm{Trace}(S^{-1} E_n) \]
Furthermore, precision errors prevented a majority of the systems from being improved as can be seen from the final column in \cref{table:newton-cg}. As a result, we passed on using the Newton correction.

After this normalization is accomplished, we perform several mathematical tricks to simplify the Monte Carlo computation. We explain details of Monte Carlo computation in \Cref{montecarlodetails}. The main take-home point is that Monte Carlo computation is highly parallelize and thus scalable whereas computer algebra alternatives do not scale well.

The reward function is essentially a partition function (as in statistical physics) that is computed in a probabilistic way. So, one cannot expect it to differentiate systems with similar number roots. We aim to use it to differentiate between systems with many roots (our target) versus the ones with few roots (initial state).  To validate, we experimented utilizing the ground truth from computational algebra software in small instances. In the process we also tuned hyperparameters introduced by the Monte-Carlo implementation. The implementation details and  hyperparameter configurations are discussed in \cref{montecarlodetails}.
The experimentation is described in \cref{montecarlohyperparameters}.

\subsection{Monte Carlo Implementation Details} \label{montecarlodetails}
We use $\mathbb{E}N(\mathcal{B}, \delta , c)$ denote the average number of solutions to the system of equations 
\[  \norm{A_1 x}^2 = c_1 \; ,  \norm{A_2x}^2 = c_2 \; , \; \cdots , \; \norm{A_nx}^2 =c_n  \]
Following the discussion  \Cref{rewarddesign}, the formula is
\[ \int_{\mathbb{R}^n} \int_{K_c(x)} \abs{\det Z^{'}(x)} \; d \mu (X_1,X_2,\ldots,X_n)  \; d x.\]
Now recall that the weight $\mu(K_c(x))$ is very small unless $x$ satisfies \Cref{Gaussloci}, that is we must have
\[ \norm{x} \sim \frac{\sqrt{n}}{\sqrt{1+\delta^2 n}} .\]
We should also note that for $A_i = C_i + \delta X_i$ to serve as decent proxy to $B_i$ we must have
\[ \delta n = \mathbb{E} \norm{\delta X_i}_F  < \norm{C_i} = 1  \Rightarrow \delta < \frac{1}{n} . \]
So, we restrict $0 < \delta \leq \frac{1}{n}$ and in this parameter regime we must have $\norm{x} \sim \sqrt{n}$ in order for $\mu(K(x))$ to be non-negligible. We will add an $ \varepsilon > \frac{4}{\sqrt{n}}$ tolerance, and we will aim to sample vectors uniformly from the annulus
\[ (1-\varepsilon) \sqrt{n}  < \norm{x} < (1+\varepsilon) \sqrt{n}  \]
Here we recall two facts about the Gaussian measure:  (1) Gaussian measure is rotationally invariant, (2) For standard Gaussian vector $x \in \mathbb{R}^n$, we have 
\[  \mathbb{P} \{ (1-\varepsilon) \sqrt{n} \leq \norm{x} \leq \frac{\sqrt{n}}{1-\varepsilon}   \}  \geq 1 - 2 e^{-\frac{\varepsilon^2n}{4}} \]
Since we fix $ \varepsilon > \frac{4}{\sqrt{n}}$ we have
\[  \mathbb{P} \{ (1-\varepsilon) \sqrt{n} \leq \norm{x} \leq \frac{\sqrt{n}}{1-\varepsilon}   \}  \geq 1 - 2 e^{-4} \geq  0.97 \]
We can use a conditional density scaled with $\frac{1}{0.97}$, but we prefer not to bother with this small detail, we'll directly use standard Gaussians to sample from the annulus
\[ (1-\varepsilon) \sqrt{n}  < \norm{x} < (1+\varepsilon) \sqrt{n} \]

After the point $x$ is sampled, we need to work on conditional expectations in Kac-Rice formula. We'll use \cite[Thm 1.1]{feliu2022kac} and standard importance sampling trick to simplify the conditional expectation expression.  More precisely, let  $A_i= B_i + \delta X_i$ where $X_i$ are $n \times n$ Gaussian matrices with independent entries with mean zero and variance $1$. These $A_i$ can be sampled using a standard random matrix generator. However, these random do not satisfy the conditioning assumption in the expression
\[  \mathbb{E} \left( \abs{det (Z^{'}(x))} \; |  \; Z(x) = (c_1,c_2,\ldots,c_n) \right) \; p_Z(c_1,c_2,\ldots,c_n)\]

To remedy this situation, we will use \cite[Thm 1.1]{feliu2022kac}. First, we pick a coordinate $x_i$ of the sampled vector $(x_1,x_2,\ldots,x_n)$. To make sure $x_i$ is neither too large or too small we can pick $x_i$ as follows:
\[  S_x =\{ \abs{x_j} : x_j^2 \geq \frac{1}{2}  \; , \;  1 \leq j \leq n \} \]
Then we pick $x_i$ to be element such that $\abs{x_i}$ is the median of $S_x$.  
After this we take a  look at our equations:
\[  \norm{A_1 x}^2 = c_i  \Leftrightarrow  x^T A_1^T A_1 x = c_i \]
\begin{equation} \label{free-KacandKevin}
    \norm{A_1 x}^2 = c_i  \Leftrightarrow    x_i^{-2} \left( c_i -  x^T A_1^T A_1 x + a^{1}_{i,i} x_i^2 \right) = a^{1}_{i,i} 
\end{equation} 
where with abuse of notation we use $a^{k}_{i,j}$ to denote $i,j$th entry of $A_k^TA_k$.
Now, we set
\[ g(A_j,c_j,x) :=   x_i^{-2} \left( c_j -  x^T A_j^T A_j x + a^j_{i,i} x_i^2 \right) \]
assuming the coordinate $i$ is already decided.   So, if we  sample all the entires of $A_1$ as before but $a^{(1)}_{i,i}$ and then condition $a^{(1)}_{i,i}$ on all the rest of the entries  using \Cref{free-KacandKevin} we can ensure $\norm{A_1 x}^2 = c_1$ is satisfied.  Thus, we do the following update:
\[ a^{j}_{i,i}   \leftarrow   g(A_j,c_j,x)  \; \text{for} \; j=1,2,\ldots,n. \]
In this case,  the ratio of densities in the  importance sampling will be determined by the update from already sampled $a^{1}_{i,i}$ to $g(A_1,c_1,x)$, that is:
\[  \exp \left(\frac{ - \left( g(A_1,c_1,x) - b^{1}_{i,i} \right)^2 +  \left(a^{1}_{i,i}-b^{1}_{i,i} \right)^2}{2} \right)   \]
All in all, our density update for the entire system in Kac-rice formula is 
\[ p(G(x)) := \prod_{j=1}^n exp \left(\frac{ - \left( g(A_j,c_j,x) - b^{j}_{i,i} \right)^2 + (a^{j}_{i,i}-b^{j}_{i,i})^2 }{2} \right)   \]
Now let us define $G = (g(A_1,c_1,x), g(A_2,c_2,x), \ldots, g(A_n,c_n,x))$ and let $D_{x} G$ denote the Jacobian matrix of this system of equations evaluated at $x$.  \cite[Thm 1.1]{feliu2022kac} shows that we then have the following equality:
\begin{equation} \label{freeKaconly}
  \mathbb{E} \left( \abs{det (Z^{'}(x))} \; |  \; Z(x) = (c_1,c_2,\ldots,c_n) \right) p_Z(c_1,c_2,\ldots,c_n) = \int \abs{\det(D_x G)}   \;  p(G(x))  \; \mu_1 \mu_2 \ldots \mu_n 
\end{equation}
where $\mu_1, \mu_2,\ldots,\mu_n$ denotes the independent Gaussian distribution that is used to sample $A_1,A_2,\ldots,A_n$ initially.
\begin{rem}
Strictly speaking \cite[Thm 1.1]{feliu2022kac} shows \Cref{freeKaconly} holds true for the case of uniform distribution on a bounded interval rather than Gaussian distribution that we've employed, however, the proof can be repeated verbatim to obtain the above equality. We do not indulge into this exercise for the sake of not adding more technicality, but a skeptical reader is welcome to do the exercise.
\end{rem}
So, our Kac-Rice formula now looks as follows:
\[  \mathbb{E}N(\mathcal{B}, c, \delta) =  \int_{\mathbb{R}^n} \mathbb{E} \left( \abs{det (Z^{'}(x))} \; |  \; Z(x) = (c_1,c_2,\ldots,c_n) \right) p_Z(c_1,c_2,\ldots,c_n)  \; d x \]
\[  \mathbb{E}N(\mathcal{B}, c, \delta) =  \int_{\mathbb{R}^n} \int \abs{\det(D_x G)}   \;  p(G(x))  \; \mu_1 \mu_2 \ldots \mu_n \; dx \]
where $\mu_1, \mu_2,\ldots,\mu_n$ denotes the independent Gaussian distribution that is used to sample the random matrices $A_1,A_2,\ldots,A_n$.

So, our  Monte Carlo approximation to $\mathbb{E}N(\mathcal{B}, c, \delta)$ where $0< \delta < \frac{1}{n}$ will go as follows:
\begin{enumerate}
\item Sample $M$ many $n$-tuples of matrices $(A_1(t),A_2(t).\cdots,A_n(t))_{t=1}^M$. $M$ is users choice, we recommend $M > n^2$.
\item Sample $N$ data points $x_1,x_2,\ldots,x_N \in \mathbb{R}^n$ independently from standard Gaussian measure on $\mathbb{R}^n$. Here $N$ will be determined by the user, it will determine the quality of approximation. We recommend $N > n^4$.
\item  Compute the following incremental sum for every matrix tuple for $t=1,\ldots,M$
\[ \frac{1}{N} \sum_{i=1}^N \abs{\det(D_{x_i} G(A_t,c_t))}   \;  \; p(G(A_t,c_t,x_i))  \]
\item  Return the following as our estimate to Kac-Rice formula
\[  \frac{1}{M} \sum_{t=1}^M  \left( \frac{1}{N} \sum_{i=1}^N \abs{\det(D_{x_i} G(A_t,c_t))}   \;  p(G(A_t,c_t,x_i)) \right)  \]
\end{enumerate}

Now let's write down $D_{x} G$ explicitly:
\[  D_{x} G = \begin{bmatrix}   
\nabla g(A_1, c_1, x) \\
\nabla g(A_2, c_2, x) \\
\vdots \\
\nabla g(A_n, c_n, x)
\end{bmatrix}
\]
where $\nabla$ denotes the gradient. Note that  
\[  \nabla g(A_j,c_j,x) =  - 2 x_i^{-2}  A_j^{T} A_j x + (c_j - x^T A_j^T A_j x)  \nabla x_i^{-2} \]
where $\nabla x_i^2 = (0,0,\ldots,0,-2 x_i^{-3},0,\ldots,0)$. Notice that $\nabla x_i^{-2} = - 2 x_i^{-2} \nabla \log(x_i)$.  So, we have
\[ 
 D_{x} G = - 2 x_i^{-2}  \left(   \begin{bmatrix}
       A_1^{T} A_1 x \\
      A_2^{T} A_2 x  \\
     \vdots \\
      A_n^{T} A_n x 
      \end{bmatrix}   
      +
      \begin{bmatrix}
           (c_1 - x^T A_1^T A_1 x)  \nabla \log(x_i) \\
           (c_2 - x^T A_2^T A_2 x)   \nabla \log(x_i)  \\
           \vdots \\
           (c_n - x^T A_n^T A_n x)  \nabla \log(x_i)
      \end{bmatrix}
      \right)
\]

\section{Experiments} \label{matrixexperiments}
To be able to verify the output of the agents with Julia Homotopy we used small size matrices of $n=10$.  We trained agents with varying episode lengths $L = 10, 15, 20$. The hyperparameter $\hat{a}$ was 0.01 for all experiments: this small step size is to make sure we are really testing agents ability to navigate the landscape even without the liberty to travel far from the initial starting point. We compare the three trained agents alongside  random sampling.
Each agent started from a uniformly generated matrix system and then its policy was used to repeatedly update the matrix system over 20 steps; this process was repeated 20 times, giving each agent 20 test runs of 20 steps.
\Cref{table:agent-averages} gives the average real solution count of all the updated matrix systems run through the agent policies.
\Cref{table:agent-exceed-successes} and \cref{table:agent-exceed-median} give how many test runs out of the 20 had an updated matrix system whose real solution count exceeded various solution counts of interest for power-flow modeling and the median number of steps the agent took to do so respectively.

\Cref{fig:agent-test-runs} gives three example test runs on the trained agents with each agent starting with the same randomly generated system.
\Cref{fig:agent-test-runs-3} shows agents steadily improving upon the starting system, especially for the $L=10, 20$ agents.
However, steady improvement does not appear to be always achieved by our agents; for example, see $L=20$ in \cref{fig:agent-test-runs-1} or $L=10$ in \cref{fig:agent-test-runs-2} where the agents struggle to stay in the $> 80$ real solution count range once quickly reaching it.
The most representative behavior seems to be that of $L=15$ in \cref{fig:agent-test-runs-2}: a general trend upwards, but one that involves both steep increases and decreases to the solution count over individual steps.

\begin{figure}[ht]
    \centering
    \begin{subfigure}[b]{0.4\textwidth}
        \centering
        \includegraphics[width=\textwidth]{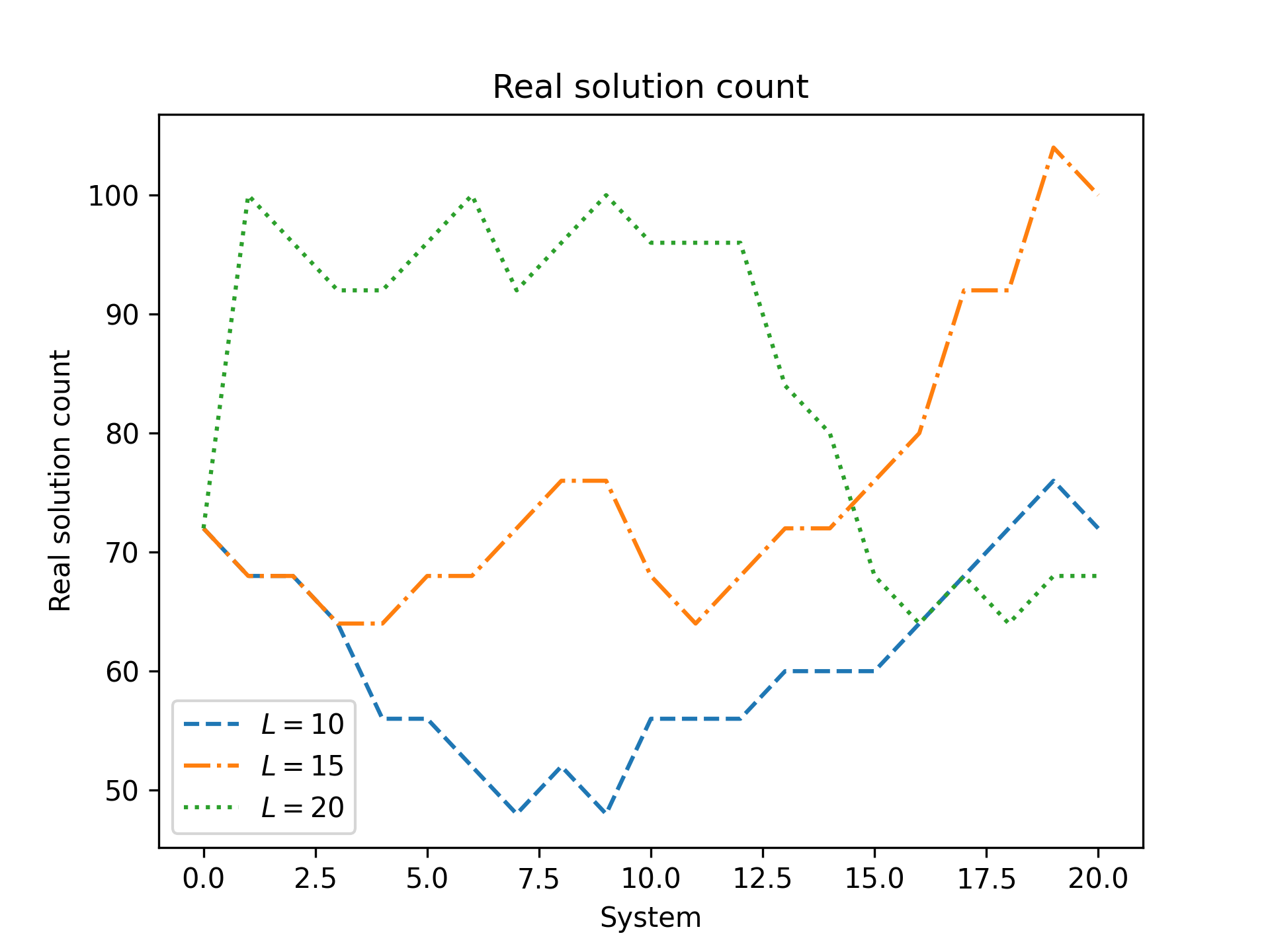}
        \caption{Test run 1}
        \label{fig:agent-test-runs-1}
    \end{subfigure}
    \hfill
    \begin{subfigure}[b]{0.4\textwidth}
        \centering
        \includegraphics[width=\textwidth]{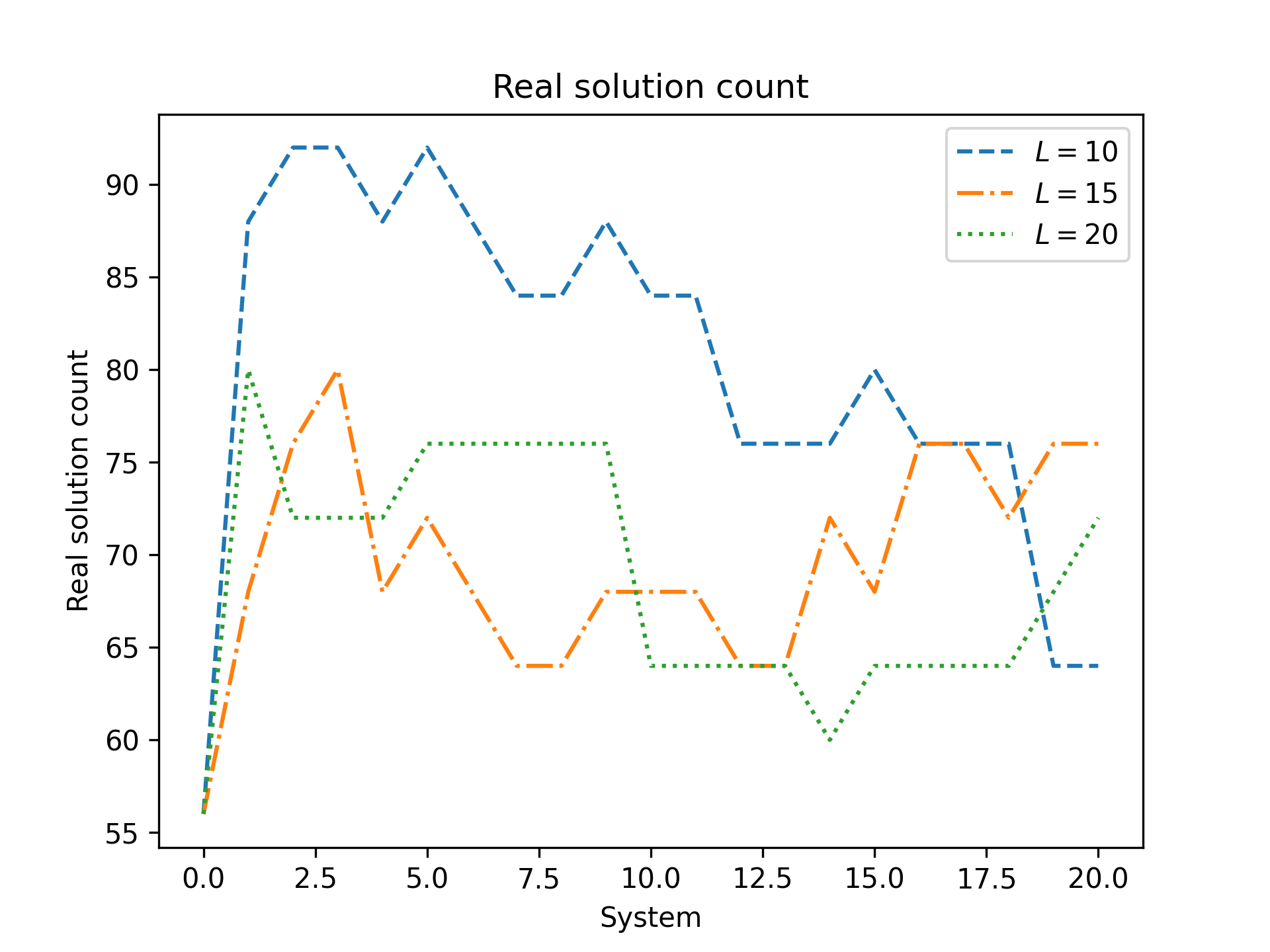}
        \caption{Test run 2}
        \label{fig:agent-test-runs-2}
    \end{subfigure}
    \hfill
    \begin{subfigure}[b]{0.4\textwidth}
        \centering
        \includegraphics[width=\textwidth]{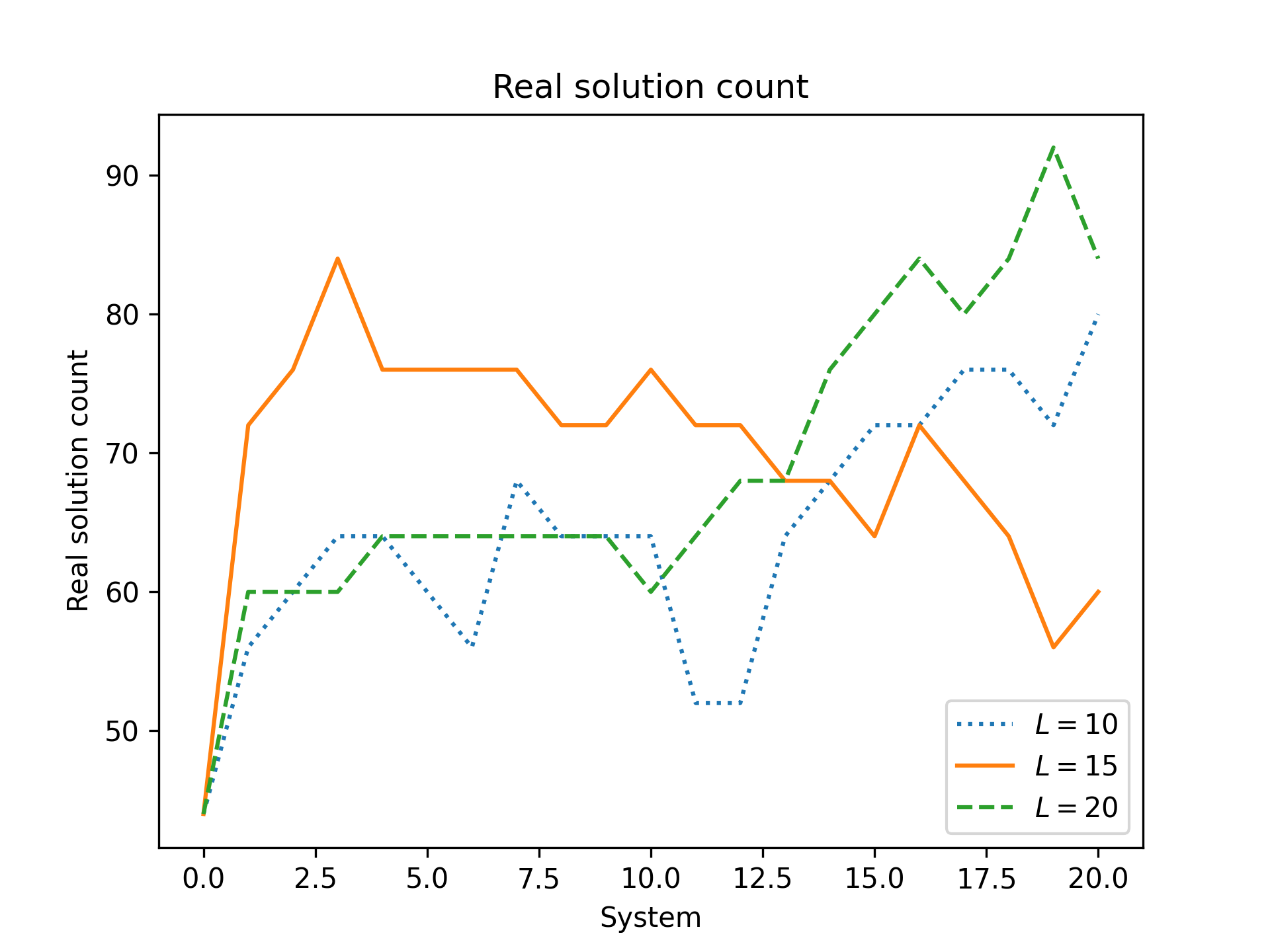}
        \caption{Test run 3}
        \label{fig:agent-test-runs-3}
    \end{subfigure}
    \caption{$L=10, 15, 20$ agent test runs}
    \label{fig:agent-test-runs}
\end{figure}

\begin{table}[ht]
    \centering
    \begin{minipage}{0.45\textwidth}
        \centering
        \begin{tabular}{llll}
        \hline
        Random & $L=10$ & $L=15$ & $L=20$ \\ \hline
        49.36  & 66.42  & 71.85  & 70.12  \\ \hline
        \end{tabular}
        \caption{Average real solution count}
        \label{table:agent-averages}
    \end{minipage}\hfill
    \begin{minipage}{0.5\textwidth}
        \centering
        \begin{tabular}{rllll}
        \hline
        Count & Random & $L=10$ & $L=15$ & $L=20$ \\ \hline
        80    & 8      & 9      & 11     & 11     \\
        90    & 4      & 6      & 7      & 7      \\
        100   & 0      & 2      & 5      & 2      \\ \hline
        \end{tabular}
        \caption{Test runs exceeding solution count}
        \label{table:agent-exceed-successes}
    \end{minipage}
\end{table}

\begin{table}[H]
    \centering
    \begin{tabular}{rllll}
    \hline
    Solution Count & Random & $L=10$ & $L=15$ & $L=20$ \\ \hline
    80                  & 12.0            & 7.0    & 6.0    & 1.0    \\
    90                  & 15.5            & 3.5    & 9.0    & 2.0    \\
    100                 & N/A             & 14.0   & 11.0   & 3.5    \\ \hline
    \end{tabular}
    \caption{Median number of steps in test runs the agents took to exceed solution counts of interest}
    \label{table:agent-exceed-median}
\end{table}
\section{Discussion}
We formulated the search for power-flow equations with many solutions as a RL problem using a novel reward function. We derived average-case behavior rigorously, and use it as a baseline. We successfully integrate several mathematical insights, convex optimization algorithms, and Monte-Carlo implementation tricks to be able implement the reward function.  The trained agents discover equations with many more solutions than what was accessible before this work; this demonstrates the potential of RL for power-flow equation modeling in particular, and for navigating complex non-linear geometries in general. We expect this will open the doors for further experimentation in power-flow modeling and for testing of conjectures in real algebraic geometry. 
\section{Acknowledgments}
A.E. and V.M are supported by NSF CCF 2414160. A.E. wants to thank Josue Tonelli Cueto for useful discussions, and to Cameron Khanpour for extra motivation to complete the paper. A.E. and V.M. are also thankful for Early Birds coffee shop where a significant portion of this work was done.

\bibliographystyle{plain}
\bibliography{prl}

\newpage
\appendix

\section{Monte-Carlo Implementation Hyperparameter Experimentation} \label{montecarlohyperparameters}
\cref{fig:evaluate-approximator} shows hyperparameter experiments for the Monte Carlo approximator on 50 systems of 10 random Gaussian matrices with independent entries centered at zero and variance 1.
For each experiment, we sampled $N=100000$ data points and $M=2500$ tuples of matrices, and then varied the value of $\delta$.
Julia Homotopy \cite{juliahomotopy} was used to calculate the true real solution count of each system, and the Monte Carlo was used to give an approximation of that.
The true and approximate counts were then normalized and sorted according to the true real solution count.
For $\delta=0.01$ in \cref{fig:evaluate-approximator-01}, three of the first ten systems have a greater approximation than the minimum approximation of the last 20 systems.
Compare this to $\delta=0.08$ where no such differentiation between systems with low and high real solution counts can be made off the approximations.

\begin{figure}[htbp]
    \centering
    \begin{subfigure}[b]{0.45\textwidth}
        \centering
        \includegraphics[width=\textwidth]{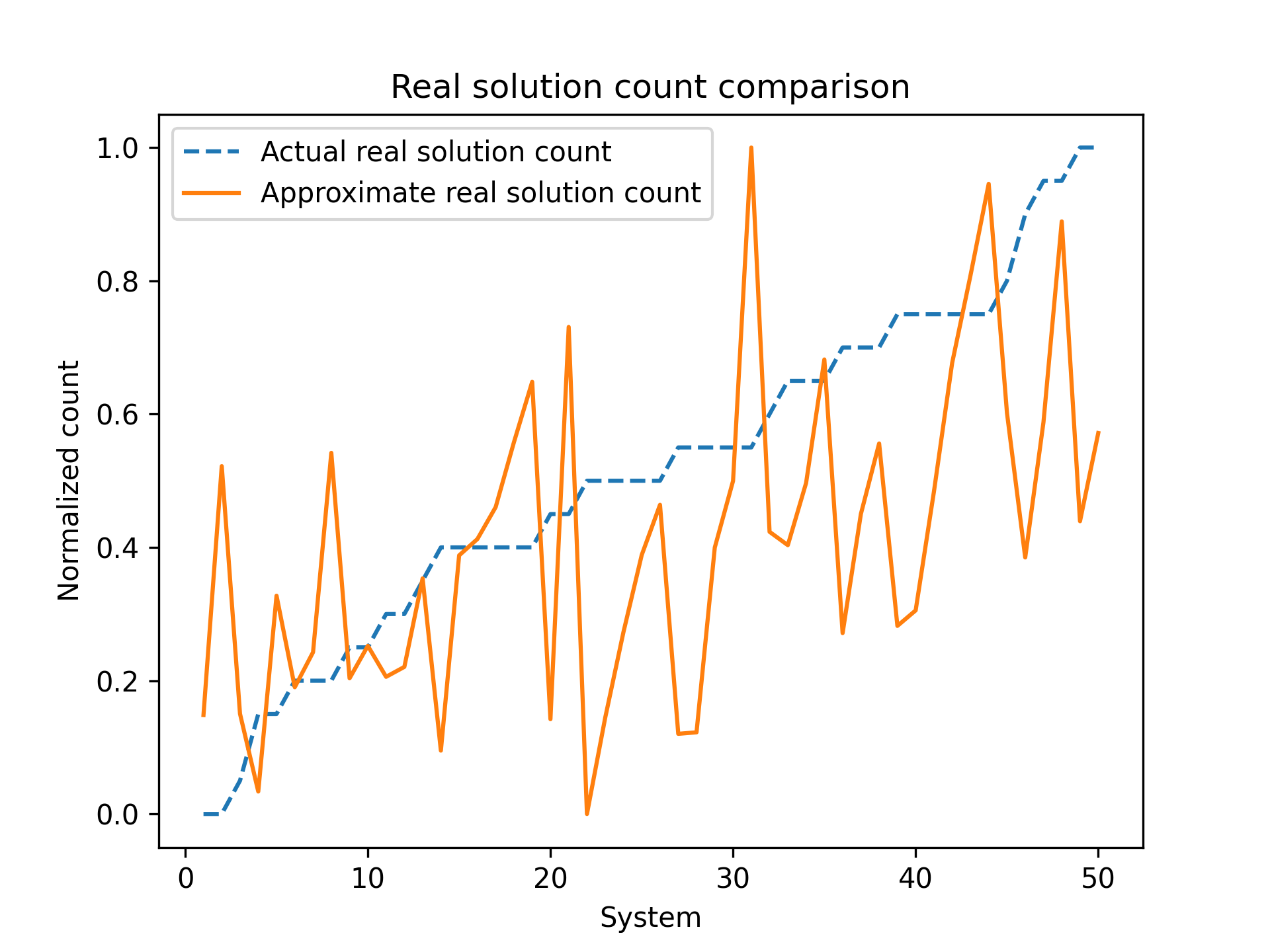}
        \caption{$\delta=0.01$}
        \label{fig:evaluate-approximator-01}
    \end{subfigure}
    \hfill
    \begin{subfigure}[b]{0.45\textwidth}
        \centering
        \includegraphics[width=\textwidth]{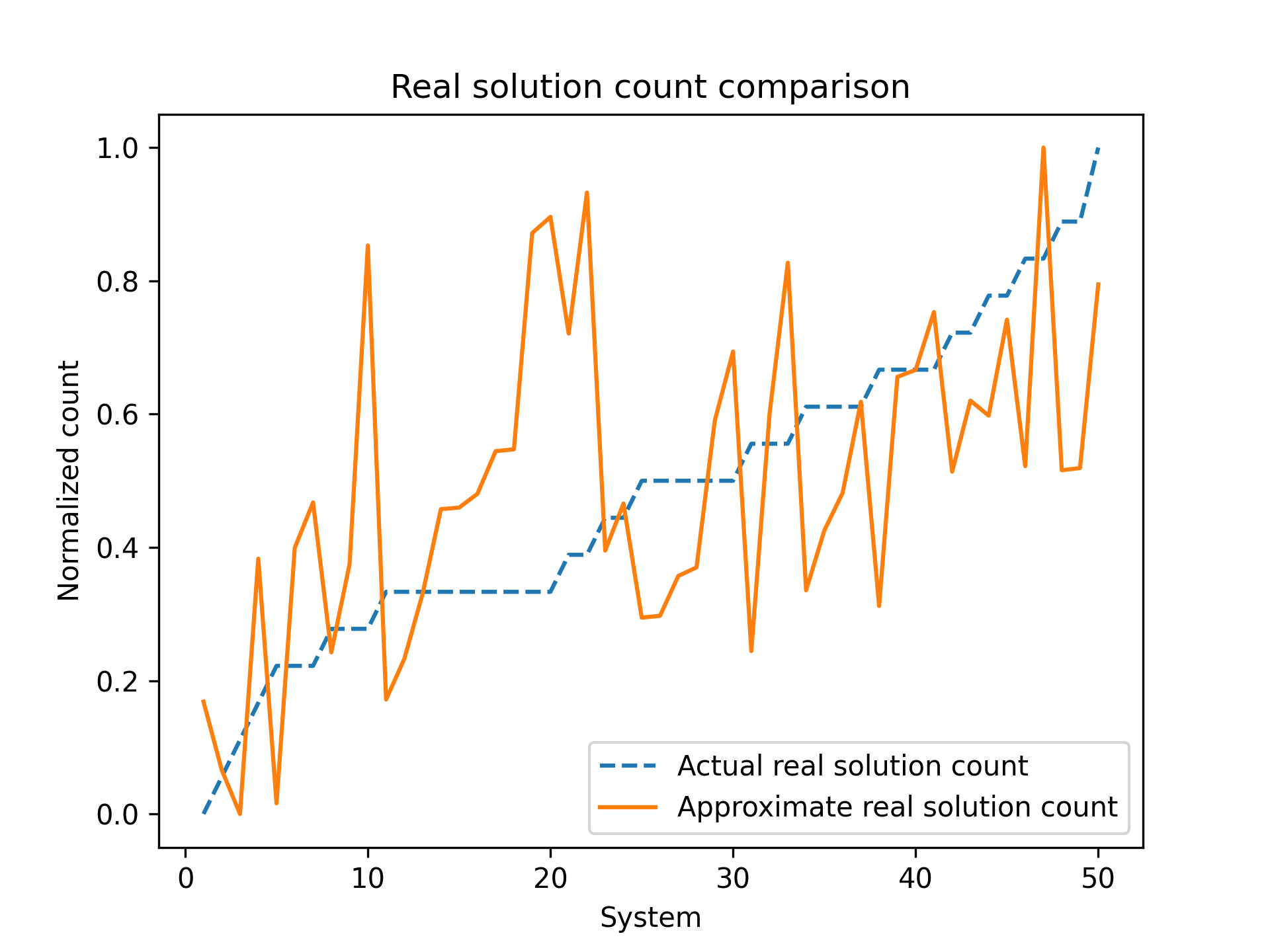}
        \caption{$\delta=0.05$}
        \label{fig:evaluate-approximator-05}
    \end{subfigure}
    \hfill
    \begin{subfigure}[b]{0.45\textwidth}
        \centering
        \includegraphics[width=\textwidth]{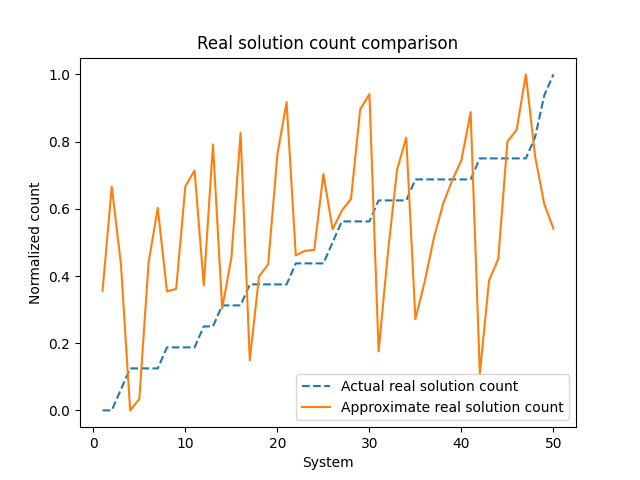}
        \caption{$\delta=0.08$}
        \label{fig:evaluate-approximator-08}
    \end{subfigure}
    \caption{Reward function comparison with $N=100000$, $M=2500$}
    \label{fig:evaluate-approximator}
\end{figure}

\section{Experimental Details}
Our code repo is available \href{https://github.com/themills22/thesis-project}{here}. The problem instance size for the state-actor space was chosen so that the matrices passed to the Monte Carlo approximator were were $10 \times 10$, meaning the Monte Carlo hyperparameters $N=100000$, $M=2500$, and $\delta=0.05$ could be reused for all the agents.
This leaves two other hyperparameters we introduced: the episode length $L$ and the action-space limit $\hat{a}$.
The episode length varied between $L=10, 15, 20$ and $\hat{a}$ was 0.01 for all agents as explained in \Cref{matrixexperiments}.
All other TD3 hyperparameters were left untouched from the Stable Baselines3 \cite{stable-baselines3} implementation of TD3.
The neural network acrhitectures used in TD3 are given in \cref{table:nn-architecture} and \cref{fig:matrix-formulation-training-plots} gives the training plots for the three agents trained.

\begin{table}[h]
\centering
\begin{tabular}{ll}
    \hline
    \textbf{Network}           & \textbf{Architecture}                                                                                 \\ \hline
    Actor $\pi_{\phi}(s)$      & Fully Connected: $D \rightarrow 500 \rightarrow 400 \rightarrow 300 \rightarrow D$, ReLU activations. \\
    Critics $Q_{\theta}(s, a)$ & Fully Connected: $D \rightarrow 500 \rightarrow 400 \rightarrow 300 \rightarrow 1$, ReLU activations. \\ \hline
\end{tabular}
\caption{Neural network architectures used where $D = 10 \times 10 \times 10$.}
\label{table:nn-architecture}
\end{table}

\begin{figure}[htbp]
    \centering
    \begin{subfigure}[b]{0.45\textwidth}
        \centering
        \includegraphics[width=\textwidth]{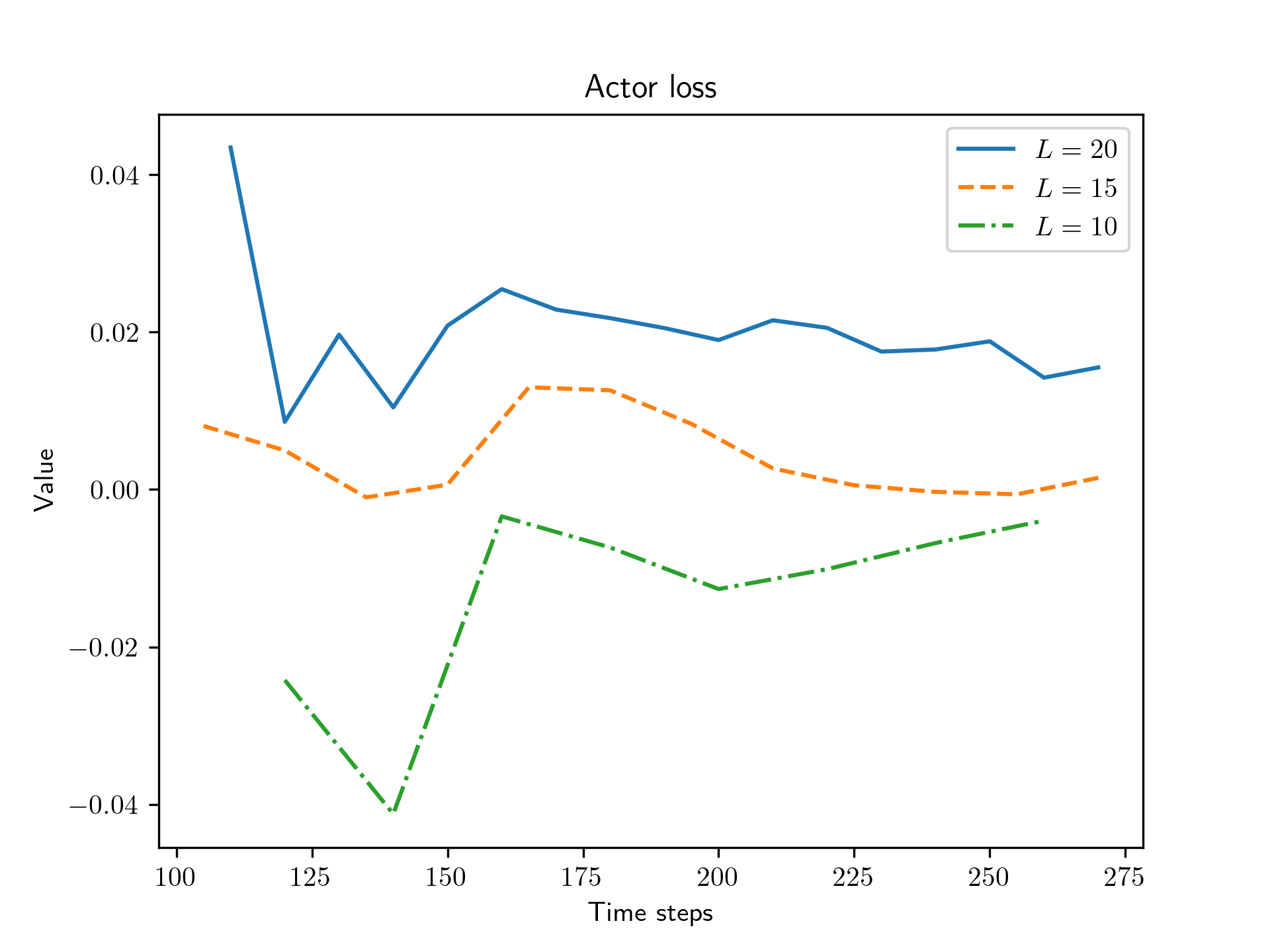}
        \label{fig:matrix-formulation-training-plots-actor-loss}
    \end{subfigure}
    \hfill
    \begin{subfigure}[b]{0.45\textwidth}
        \centering
        \includegraphics[width=\textwidth]{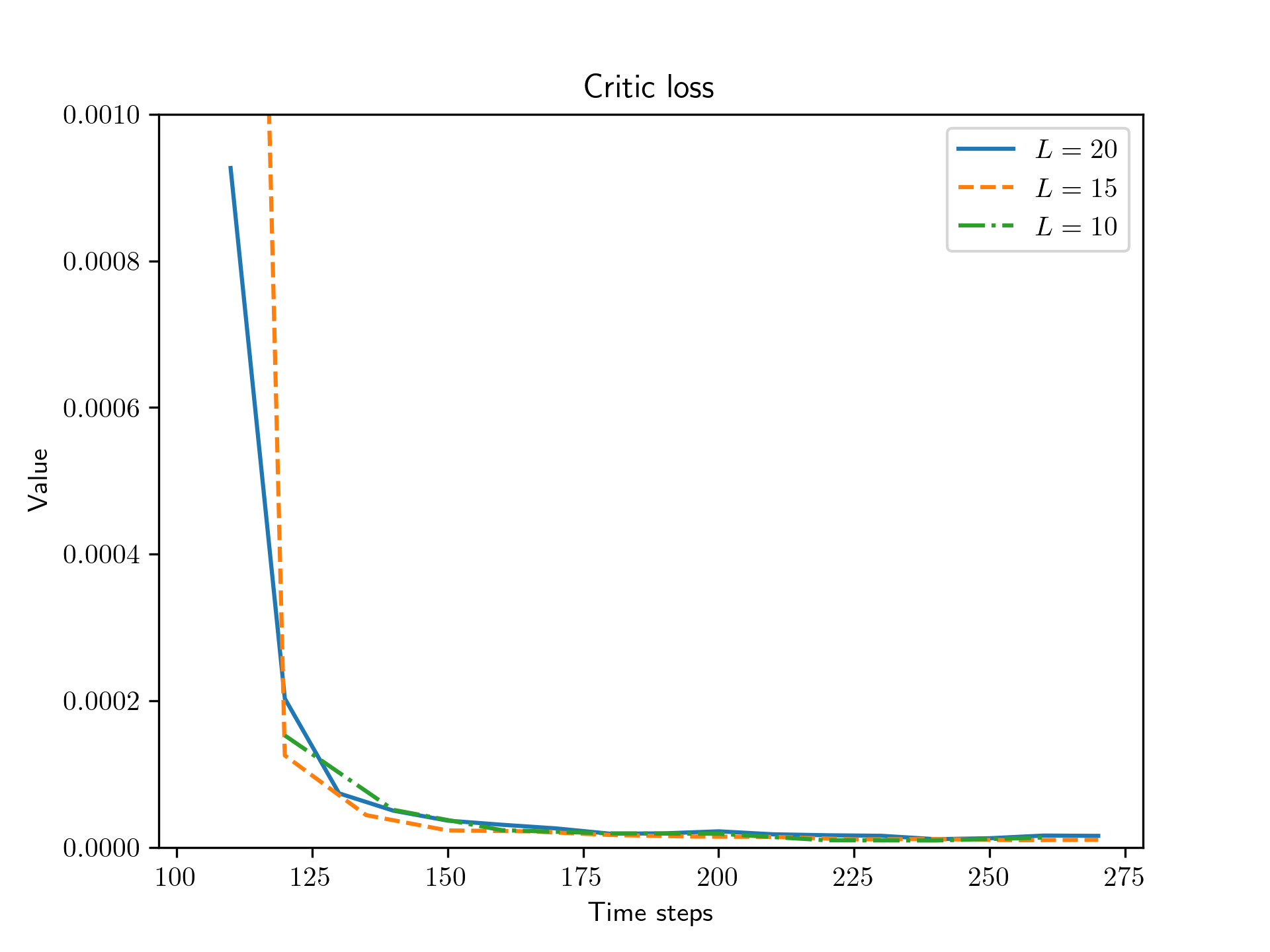}
        \label{fig:matrix-formulation-training-plots-critic-loss}
    \end{subfigure}
    \hfill
    \begin{subfigure}[b]{0.45\textwidth}
        \centering
        \includegraphics[width=\textwidth]{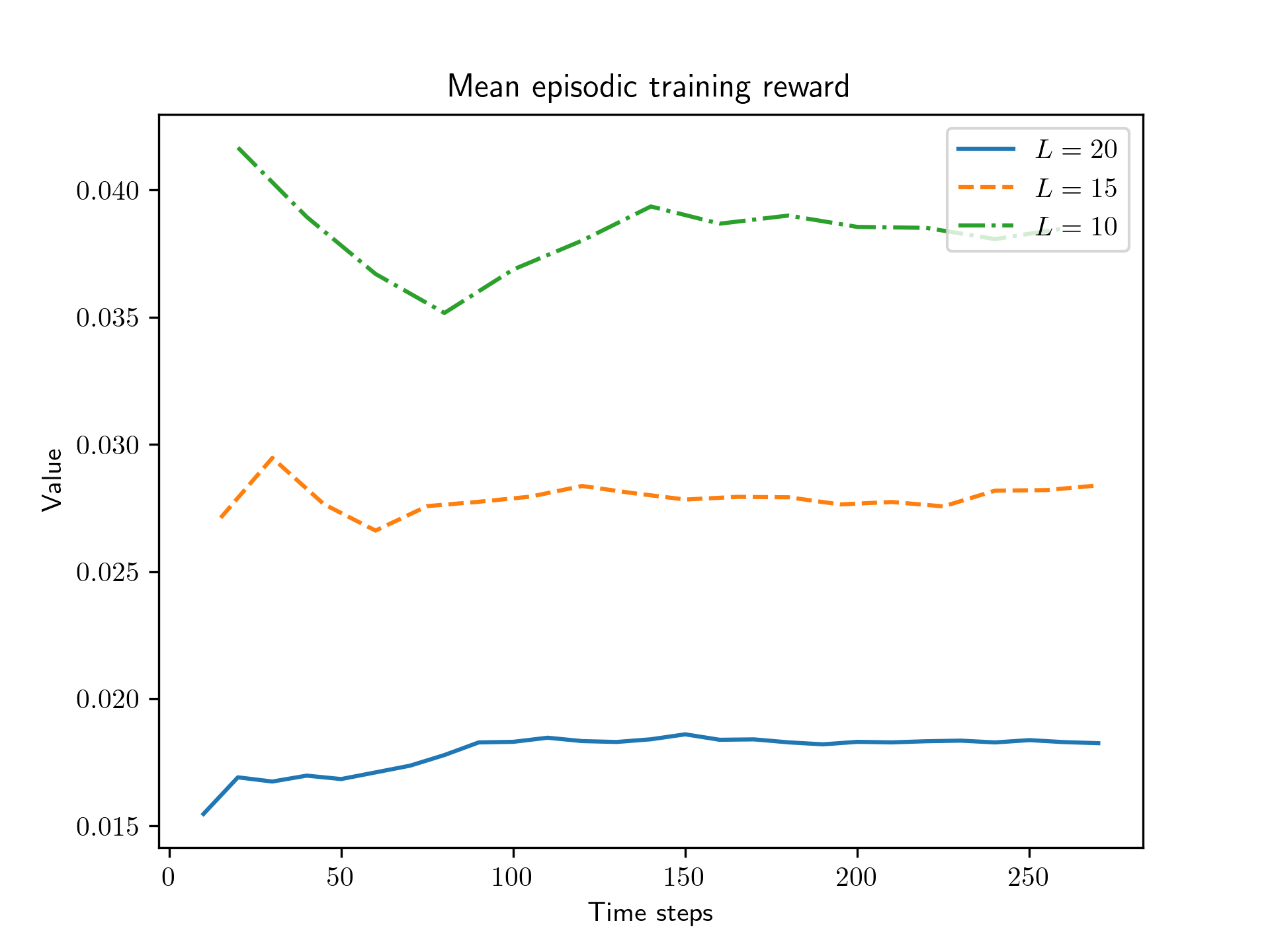}
        \label{fig:matrix-formulation-training-plots-mean-reward}
    \end{subfigure}
    \caption{Training plots for the $L=10, 15, 20$ matrix formulation agents.}
    \label{fig:matrix-formulation-training-plots}
\end{figure}

\section{Matrix form of power flow equations in Example \ref{ex: pf eqs}} \label{detailedeq}
These equations can be written in matrix form as 
\begin{align*}
    P_1 &= \begin{bmatrix}
        x_0 \\ x_1 \\ x_2 \\ y_0 \\ y_1 \\ y_2
    \end{bmatrix}^T \begin{bmatrix}
        0 & 0 & 0 &0 & \frac{b_{01}}{2} & 0  \\
        0 & 0 & 0 & - \frac{b_{01}}{2} & 0 & - \frac{b_{12}}{2} \\ 
        0 & 0 & 0 & 0 & \frac{b_{12}}{2} & 0 \\
         0 & - \frac{b_{01}}{2} & 0 & 0 & 0 & 0 \\
         \frac{b_{01}}{2} & 0 & \frac{b_{12}}{2} & 0 & 0 & 0 \\
         0 &  - \frac{b_{12}}{2} & 0 & 0 & 0 & 0
    \end{bmatrix} \begin{bmatrix}
        x_0 \\ x_1 \\ x_2 \\ y_0 \\ y_1 \\ y_2
    \end{bmatrix} \\
    P_2 &= \begin{bmatrix}
        x_0 \\ x_1 \\ x_2 \\ y_0 \\ y_1 \\ y_2
    \end{bmatrix}^T \begin{bmatrix}
        0 & 0 & 0 &0 & 0 & 0  \\
        0 & 0 & 0 & 0 & 0 &  \frac{b_{12}}{2} \\ 
        0 & 0 & 0 & 0 & -\frac{b_{12}}{2} & 0 \\
         0 & 0 & 0 & 0 & 0 & 0 \\
        0 & 0 & -\frac{b_{12}}{2} & 0 & 0 & 0 \\
         0 &   \frac{b_{12}}{2} & 0 & 0 & 0 & 0
    \end{bmatrix} \begin{bmatrix}
        x_0 \\ x_1 \\ x_2 \\ y_0 \\ y_1 \\ y_2
    \end{bmatrix} \\
    1 &= \begin{bmatrix}
        x_0 \\ x_1 \\ x_2 \\ y_0 \\ y_1 \\ y_2
    \end{bmatrix}^T \begin{bmatrix}
        0 & 0 & 0 & 0 & 0 & 0 \\
        0 & 1 & 0 & 0 & 0 & 0 \\
        0 & 0 & 0 & 0 & 0 & 0 \\
        0 & 0 & 0 & 0 & 0 & 0 \\
        0 & 0 & 0 & 0 & 1 & 0 \\
        0 & 0 & 0 & 0 & 0 & 0
    \end{bmatrix}\begin{bmatrix}
        x_0 \\ x_1 \\ x_2 \\ y_0 \\ y_1 \\ y_2
    \end{bmatrix} \\
    1 &= \begin{bmatrix}
        x_0 \\ x_1 \\ x_2 \\ y_0 \\ y_1 \\ y_2
    \end{bmatrix}^T \begin{bmatrix}
        0 & 0 & 0 & 0 & 0 & 0 \\
        0 & 0 & 0 & 0 & 0 & 0 \\
        0 & 0 & 1 & 0 & 0 & 0 \\
        0 & 0 & 0 & 0 & 0 & 0 \\
        0 & 0 & 0 & 0 & 0 & 0 \\
        0 & 0 & 0 & 0 & 0 & 1
    \end{bmatrix}\begin{bmatrix}
        x_0 \\ x_1 \\ x_2 \\ y_0 \\ y_1 \\ y_2
    \end{bmatrix} \\
    1 &= \begin{bmatrix}
        x_0 \\ x_1 \\ x_2 \\ y_0 \\ y_1 \\ y_2
    \end{bmatrix}^T \begin{bmatrix}
        1 & 0 & 0 & 0 & 0 & 0 \\
        0 & 0 & 0 & 0 & 0 & 0 \\
        0 & 0 & 0 & 0 & 0 & 0 \\
        0 & 0 & 0 & 1 & 0 & 0 \\
        0 & 0 & 0 & 0 & 0 & 0 \\
        0 & 0 & 0 & 0 & 0 & 0
    \end{bmatrix}\begin{bmatrix}
        x_0 \\ x_1 \\ x_2 \\ y_0 \\ y_1 \\ y_2
    \end{bmatrix} \\
    1 &= \begin{bmatrix}
        x_0 \\ x_1 \\ x_2 \\ y_0 \\ y_1 \\ y_2
    \end{bmatrix}^T \begin{bmatrix}
        1 & 0 & 0 & 0 & 0 & 0 \\
        0 & 0 & 0 & 0 & 0 & 0 \\
        0 & 0 & 0 & 0 & 0 & 0 \\
        0 & 0 & 0 & 0 & 0 & 0 \\
        0 & 0 & 0 & 0 & 0 & 0 \\
        0 & 0 & 0 & 0 & 0 & 0
    \end{bmatrix}\begin{bmatrix}
        x_0 \\ x_1 \\ x_2 \\ y_0 \\ y_1 \\ y_2
    \end{bmatrix}
\end{align*}

\end{document}